\documentclass[letterpaper]{article}
\pdfoutput=1

\usepackage{times}
\usepackage{helvet}
\usepackage{courier}
\usepackage[hyphens]{url}
\usepackage{graphicx}
\usepackage{natbib}
\usepackage{caption}

\usepackage{amsmath}
\usepackage{dsfont}
\usepackage{amsfonts}
\usepackage{amsthm, thmtools}
\usepackage{algorithm}
\usepackage[noend]{algpseudocode}
\usepackage{natbib}
\usepackage[switch]{lineno}
\usepackage{microtype,subcaption}
\usepackage[capitalize,nameinlink]{cleveref}

\usepackage[colorinlistoftodos]{todonotes}

\usepackage[normalem]{ulem}
\usepackage{authblk}

\theoremstyle{plain}
\newtheorem{thm}{Theorem}[section]
\newtheorem{lem}[thm]{Lemma}

\theoremstyle{definition}
\newtheorem{defn}{Definition}[section]

\theoremstyle{remark}

\theoremstyle{casep}
\newtheorem{casep}{Case}

\DeclareMathOperator*{\argmax}{arg\,max}

\newcommand{\E}{\mathop{\mathbb E}}

\title{LISPR: An Options Framework for Policy Reuse with Reinforcement Learning}

\date{}

\author[1]{Daniel Graves}
\author[1,2]{Jun Jin}
\author[1]{Jun Luo}
\affil[1]{Noah's Ark Lab, Huawei Technologies Canada, Ltd., Canada}
\affil[2]{Computing Science Department, University of Alberta, Canada}

\begin{document}

\maketitle

\begin{abstract}
We propose a framework for transferring any existing policy from a potentially unknown source MDP to a target MDP.
This framework (1) enables reuse in the target domain of any form of source policy, including classical controllers, heuristic policies, or deep neural network-based policies, (2) attains optimality under suitable theoretical conditions, and (3) guarantees improvement over the source policy in the target MDP.
These are achieved by packaging the source policy as a black-box option in the target MDP and providing a theoretically grounded way to learn the option's initiation set through general value functions.
Our approach facilitates the learning of new policies by (1) maximizing the target MDP reward with the help of the black-box option, and (2) returning the agent to states in the learned initiation set of the black-box option where it is already optimal.
We show that these two variants are equivalent in performance under some conditions.
Through a series of experiments in simulated environments, we demonstrate that our framework performs excellently in sparse reward problems given (sub-)optimal source policies and improves upon prior art in transfer methods such as continual learning and progressive networks, which lack our framework's desirable theoretical properties.
\end{abstract}


\section{Introduction}
One of the most important challenges in applying reinforcement learning (RL) to real-world problems is learning policies that work well to variations of the training task \citep{whiteson2011, zhao2019, henderson2017, henderson2017rlmatters, akkaya2019solving}.
Unlike game environments that offer well-defined tasks, real-world tasks are much harder to encapsulate.
Delineating the task too narrowly limits the applicability and generality of the learned solution.
Delineating the task too broadly makes it inefficient if not impossible to learn \textit{tabula rasa} -- without any prior knowledge of the task environment.
Transferring policies from a source to a target domain is an important technique for addressing this challenge \citep{taylor2009, taylor2011, czarnecki2018}.
Our goal is to present a theoretically grounded framework for reusing and improving upon prior knowledge in the form of available policies.

\begin{figure}
    \centering
    \includegraphics[width=0.2\textwidth]{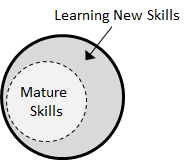}
    \caption{LISPR learns a new policy in the zone of proximal development while making use of a mature policy.}
    \label{fig_zone_of_proximal_development}
\end{figure}

We draw inspiration from developmental psychology, where the \textit{zone of proximal development} \citep{vygotsky1978} refers to a state where the child already has relevant skills but needs extra support to perform a task as shown in Figure \ref{fig_zone_of_proximal_development}.
A 2-year-old may be able to walk on the floor and crawl up stairs, but still needs a caregiver to hold their hands to walk up stairs.
This same child may also know how to unroll toilet paper and spread it all over the floor, but is yet to learn to tear off just as much as needed.
With the social context setting up appropriate learning opportunities for such a child in the zone of proximal development, they will further develop their prior skills to eventually solve these new problems independently.

The concept of options \citep{sutton1999} provides an excellent foundation for transfer prior skills to new challenges \citep{barreto2019, yang2020ptf}.
From the perspective of options, an agent in the zone of proximal development may be viewed as in the process of developing a well-orchestrated set of options, wherein a subset of mature but limited options could be enlisted to both provide a basic level of performance and facilitate the transfer process.

To turn this line of thinking into a concrete transfer learning framework, we propose the following specific method.
First, (1) the source policy is packaged up as an option, which we call the \textit{primal option}.
Second, (2) the initiation set of the primal option is learned in the target domain using general value functions (GVFs) \citep{sutton2011}.
Third, (3) a \textit{learner option} is introduced, which may be either a \textit{recovery option} or a \textit{student option} according to whether it is to learn to support the reuse of the primal option or to learn from it towards superseding it.
Lastly, (4) a higher-level \textit{main policy} is used to control the choice between the primal option and the learner option based on the learned initiation set of the primal option.
Because the core of this framework lies in learning the initiation set, we call it \textit{Learning Initiation Set for Policy Reuse} (LISPR).

In the rest of this paper, we (a) explain the LISPR framework in detail, (b) prove theorems about its optimality and performance guarantees, and (c) show how it flexibly supports both the reuse of a source policy as is and the use of it as a tutor for a stronger student option.  We use experiments to demonstrate (i) how LISPR allows us to use RL to transfer non-learning-based heuristic policies, (ii) how it makes sparse reward problems easier to solve, and (iii) how it exhibits more consistent learning performance than methods of feature transfer.
\section{Related Works}

\subsubsection{Learning Initiation Sets}
Despite growing attention to option discovery, research tend to focus on learning policies and termination functions \citep{machado2017, harutyunyan2019}, rather than learning initiation sets of options \citep{demir2019, barreto2019, khetarpal2020optionsinterest}.
The Option Keyboard \citep{barreto2019} provides a way to learn the initiation set of a class of options defined by cumulants by adding a special action called the termination action that terminates the option.
The Options of Interest approach \citep{khetarpal2020optionsinterest} learns a generalization of initiation sets defined by interest functions and derives an algorithm called the interest-option critic to learn these functions.
Heuristic goal-oriented methods of generating initiation sets are proposed in \citep{demir2019} for simple tabular domains.
In contrast to these approaches, LISPR learns initiation sets through learning success predictors that are essentially value functions.
This is why LISPR naturally has its desirable theoretical properties and guarantees.
In \citep{mcgovern2001}, authors suggest that frequently visited states could be used as sub-goals for options, which has motivated more recent work such as \citep{machado2017}.
While visitation frequency can also be used as a heuristic to identify initiation sets \citep{stolle2002} \citep{demir2019}, estimates of state visitation lead to divide and conquer strategies that may be challenging to scale to large problems.
In contrast, LISPR allows performance of the overall system to continue to improve as new skills are learned similar to what happens in zones of proximal development \citep{vygotsky1978} when new neural activity continues \citep{oby2019} to develop.

\subsubsection{Policy Transfer \& Policy Reuse}
Recently, the policy transfer framework (PTF) \citep{yang2020ptf} introduced a way to use the options framework to set up policy distillation for transferring source policies to the target policy.
PTF learns the termination function of the options by maximizing the expected return but the initiation set is assumed to be universal.
Like PTF, LISPR does not need to measure the similarity between the source and target MDPs.
Source policies have been used to explore more efficiently towards an optimal target policy \citep{siyuan2017sourcepolicybandit, kurenkov2019acteach, siyuan2019contextawarepolicyreuse}.
For example, \cite{siyuan2017sourcepolicybandit} reduce the problem of selecting a source policy for exploration to a multi-armed bandit problem where a wide body of research provides theoretically grounded ways of exploring.
Our proposed LISPR framework differs in that it formulates the problem of policy transfer to learning initiation sets of options in the target MDP.
This not only allows transfer through distillation and exploration as is the case with \textit{learning to master} (Section \ref{sec_learning_to_master}) but also enables transfer through \textit{learning to recover} (Section \ref{sec_learning_to_recover}) when the source policy performance falls short in the new task.
Learning to recover could be highly valuable if exploration in the part of the state space covered by the source policy is prohibitively dangerous or expensive.
In \citep{kurenkov2019acteach}, multiple teachers are used to improve learning while avoiding negative transfer.
This is achieved by learning a Bayesian value function with heuristic multi-step exploration.  LISPR avoids negative transfer differently by learning the initiation set of the source policy such that the policy is only executed when it is useful.
An underyling theme in most policy transfer and policy reuse works is the lack of theoretically grounded frameworks to support and motivate their use.
\cite{siyuan2019contextawarepolicyreuse} provides theory for convergence and optimality but is limited to discrete action policies and requires that the source policies must be optimal; LISPR is applied to continuous action spaces and does not require the source policies be optimal.

\cite{fernando2010probpolicyreuse} developed probabilistic policy reuse that chooses policies stochastically using a Boltzmann distribution over the returns.
They apply their method to RoboCup keepaway problem showing speed up in learning; no supporting theory is provided.
A simple method to transfer a PID controller to a target domain is shown in \cite{xie2018trainingwheels}.
DDPG was used to learn a second policy and DQN was applied to choose between the PID controller and the DDPG policy.
They found their method improved exploration significantly.
LISPR on the other hand replaces DQN with a simple heuristic that we show is optimal.
While most methods focus on multi-task transfer, LISPR focuses on single task transfer which is arguably a more practical use case for transfer learning.

\subsubsection{Feature Transfer}
Perhaps the most common methods of transfer with RL is feature transfer.
A dominant approach is to initialize the weights of a policy with prior knowledge (e.g. supervised learning to pre-train AlphaGo \citep{silver2016}) for later fine tuning.
However, catastrophic forgetting \citep{kirkpatrick2017} is a common issue.
Progressive neural networks \citep{rusu2017} address this issue by injecting features of the source policy into each layer of the new policy to speed up learning.
But this presupposes a deep neural network architecture which rules out hybrid solutions combining classical and learned policies.
In addition, these methods of transfer do not offer theoretical guarantees.
Learning reset policies could also improve learning by avoiding costly resets and methods of thresholding a value function similar to that of LISPR is common in learning reset policies \citep{eysenbach2018reset}.
But the thresholding is assumed to be a constant and has no supporting theory behind it.
In contrast, LISPR is more general because the threshold is adaptive and is theoretically well-grounded.

Successor features \cite{barreto2017successorfeatures} is a popular transfer learning framework that developed out of an older idea called successor representation.
The underlying idea is to exploit a linear relationship between reward and features in multiple related problems by learning feature predictors.
However, the assumption is that the same features can be linearly combined in different ways to produce different rewards.
This works well in problems that are closely related but in practice, this assumption can be quite strong since when adding complexity that the agent has never seen before.
The features may not be expressive enough to capture the complexity in some target domains.
In addition, successor features is discrete action; we focus on transferring policies in continuous action spaces.
The advantage of LISPR is that it permits transfer of any black-box policy while all feature transfer methods do not permit transfer of non-neural network source policies.

\subsubsection{Curriculum Learning}
LISPR has some apparent connections with curriculum learning \citep{czarnecki2018}.
Consider the case where $\tau(s)$ is set to some parameter $\lambda$.
From Theorem \ref{thm_recovery_optimality} (and similarly Theorem \ref{thm_student_optimality}), it is clear that the main policy (Definition \ref{def_main_policy} is not optimal unless $\lambda \geq \sup_{\forall s \in S}{V^{R}(s)}$.
Otherwise, the main policy will never select the primal option and thus the LISPR framework reduces to standard reinforcement learning.
This presents an interesting idea of creating a schedule $\lambda(t)$ such that $\lambda(t)$ increases gradually over time until $\lambda \geq \sup_{\forall s \in S}{V^{M}(s)}$ is satisfied.
The result is that while the recovery policy is being trained, the size of the initiation set $L$ slowly reduces to $\emptyset$.
We hypothesize that this constitutes a curriculum learning setup because small values of $\lambda$ present an easier problem for the recovery policy compared to larger values of $\lambda$, because for smaller values of $\lambda$ the initiation set $L$ is larger.
This is trivial to show from Definition \ref{def_initiation_set} where $L_2(s) \subset L_1(s)$ if $\lambda_2 > \lambda_1$.
Therefore, LISPR could be viewed as doing ``self-paced'' curriculum learning if $\tau(s)$ gradually increases over the course of learning as one might expect to be the case for $\tau(s)=V^{R}(s)$ given appropriate initialization of $V^{R}(s)$.
A second connection with curriculum learning has to do with the relationship between the source and target MDPs.
Choosing a source policy that transfers well to the target MDP is a challenging problem.
We hypothesize that a good choice of source policy would likely constitute some sort of curricular design.
Our \textbf{bipedal walker task} is an example of such curricular design where the source policy is learned without pits, stairs, or obstacles before transferring to the harder task where these are all present.

\subsubsection{Scaffolding}
Another idea that is closely related to the LISPR framework is scaffolding \citep{wood1976}.
Scaffolding involves a tutor that provides guidance to help the agent solve the problem.
If the skill to transfer is an expert policy such as a heuristic policy, then LISPR can be considered a kind of \textit{scaffolding framework} where the expert policy takes the role of a tutor \citep{wood1976}.
However, the LISPR framework is more general than scaffolding because all prior skills learned by the agent on previous tasks can also be considered source policies to be reused again.
Moreover, the agent's performance is not bounded by the performance of the expert as the theory guarantees continuous improvement towards optimality.

\section{The LISPR Framework}

\begin{figure}
    \centering
    \includegraphics[width=0.45\textwidth]{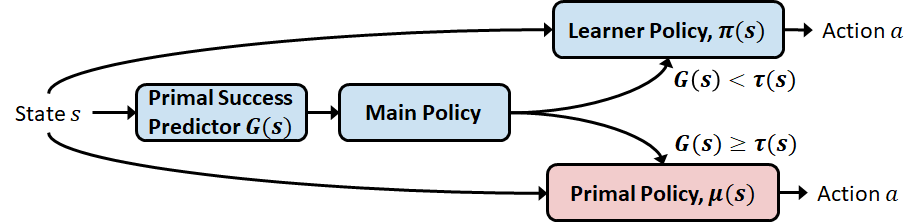}
    \caption{LISPR framework learns $G(s)$ which predicts the success of the primal policy and is used to construct its initiation set.  The main policy decides whether to execute the primal policy or learner policy based on the initiation set of the primal policy.}
    \label{fig_framework}
\end{figure}

The central idea of LISPR is to evaluate the performance of an existing policy -- or source policy -- in a new task for the purposes of understanding when to use the source policy and when to learn a new policy in order to achieve the best possible performance.
LISPR achieves this by predicting the success or failure of the source policy according to its current state and using this to define the set of states where the source policy can be initiated for best performance as shown in Figure \ref{fig_framework}.
Consequently, the agent is able to focus on learning a policy that either assists the source policy by returning to states where it is successful or learns from the source policy to supersede it.

\subsection{Problem Setting}

\subsubsection{MDP}
An MDP is defined by the tuple $(S, A, P, r, \gamma)$, where $S$ is the set of states, $A$ is the set of actions, $P$ is the state transition dynamics modeled as one-step transition probabilities, $r$ is the reward, and $\gamma$ is the discount factor.
The goal of RL is to learn a policy $\pi(a|s)$ to maximize expected return, i.e. $\pi* = \text{argmax}_{\pi}\E_{\pi, P}{[R_t]}$, where the return $R_t$ is the future discounted sum of rewards $R_t=\sum_{i=0}^{\infty}{\gamma^i r_{t+i+1}}$.
The optimal policy $\pi*$ can be considered the solution to the MDP, but in practice it is often only approximated.
The expected return starting from a state is called the value of that state, i.e. $V^\pi(s_t)=\E_{\pi, P}{[R_t]}$.
RL is often assumed to happen with no prior knowledge of the MDP it is solving. The agent is taken to be a blank slate, or \textit{tabula rasa}, when learning starts.

\subsubsection{Policy Transfer}
Given an MDP $A$ and knowledge about a policy $\pi^A$ learned in $A$, policy transfer is about using the knowledge of $\pi^A$ to find a policy that maximizes value in $B$ \citep{taylor2009}.
One approach is to copy the policies, i.e. $\pi^B=\pi^A$, in which case $\pi^B$ is not a solution to $B$ but the tasks must be similar enough to be considered an approximation to it.
Another approach is to modify $\pi^A$ to find a better policy $\pi^B$.
This may involve modifying weights of a neural network if the policy is a deep policy.
A third approach is to measure the similarity between $A$ and $B$ so as to modify $\pi^A$ to get $\pi^B$.
The purpose of policy transfer is simply to speed up learning in $B$.

\subsubsection{Conditions for Transfer}
In policy transfer, there are often conditions on the kinds of policies that can be transferred from source to target MDPs.
While many policy transfer methods assume shared dynamics, such as successor features \citep{barreto2017successorfeatures}, these conditions are quite strong and don't permit transfer from simple problems to more complex ones that require more state information.
LISPR is more flexible where the state space of the source MDP could be a sub-set of the state space of the target MDP.
More generally, LISPR requires that there exists at least one mapping from the state space of the source MDP to the state space of the target MDP; as an example, this may include adding new features to the state space of the target MDP.
Successor features, for example, doesn't allow this since the features of the state space of the target MDP are constrained to be the same features of the source MDP. 
Similarly, LISPR requires that there exists a mapping from the action space of the source MDP to the action space of the target MDP.
In addition, since LISPR learns the initiation set of the source policy, it can quickly learn if the source policy hinders learning, such as if the policy is a contradictory policy described in \citep{kurenkov2019acteach}, and shrink the initiation set.
LISPR allows initiation sets that are empty by design which means that in the case of a contradictory policy, the initiation set will collapse to the empty set and not be used.
While not a condition, LISPR will perform better when the source policy is able to complete the task, but it imposes no conditions on the optimality of the source policy.

\subsection{From Source Policy to Primal Option}

Options are a way to temporally extend actions as policies in an MDP that have well defined initiation and termination conditions \citep{sutton1999}.
They are temporal abstractions of lower-level or more fine-grained actions.
Formally, an option is the tuple $(\pi, \mathcal{I}, \beta)$, where $\pi(a|s)$ is the policy of the option, $\mathcal{I} \subseteq S$ is the initiation set of the option, and $\beta(s)$ is the termination probability function of the option.
In practice, $\beta(s)$ can be given or learned,  $\pi(a|s)$ can be given or learned, and $\mathcal{I}$ can be given but is typically assumed to be the entire state space $S$.
Crucially, the LISPR framework takes exception and provides a principled way to learn $\mathcal{I}$ of a given policy for the purposes of policy transfer.

Consider the case of transferring a mature (i.e. immutable) source policy $\mu$ from a source task, the environment for which is unavailable.
The agent is challenged with a new task, defined as an MDP $M$, and the objective is to exploit the mature source policy to quickly learn a new policy that maximizes the value for $M$.
For that, we need to package $\mu$ up as an option that is usable for $M$. Specifically, we first learn a success predictor of $\mu$ in $M$ and then use this predictor to construct the initiation set of the option.

\begin{defn}[Success Predictors]
The success predictor is the future accumulated success of any policy $\mu(a|s)$ in the target MDP $M$ denoted by the general value function $G^\mu(s)=\E_{\mu}{[\sum_{i=0}^{\infty}{\gamma^i r_{t+i+1}}|s_t=s]}$ for all states $s \in S$ where $r_{t}$ is the reward at time $t$.
\label{def_success}
\end{defn}

The source policy $\mu$ may be any black-box policy, e.g. a deep network policy, a classical controller, a heuristic policy, etc.
Learning $G(s,a)$ permits off-policy methods with standard temporal difference learning where $G(s)=G(s,a) \rvert_{a \sim \mu}$ \citep{sutton1998}.
Since the policy $\mu$ is mature or fixed, $G(s)$ is usually quick to learn \citep{graves2020}.
Because the success predictor estimates the value of policy $\mu$ in MDP $M$ if initiated in any state in $S$, the learned success predictor can be used to construct an initiation set of policy $\mu$ as the set of all states where the policy $\mu$ equals or exceeds a baseline performance.

\begin{defn}[Initiation Set]
Given MDP $M$, the initiation set $L$ of policy $\mu$ is defined with the success predictor $G^\mu(s)$ such that $L=\{s|s \in S \text{ such that } G^\mu(s) \geq \tau(s)\}$ for some baseline threshold function $\tau(s)$.
\label{def_initiation_set}
\end{defn}

$L$ is the set of all states that leads to predicted success based on the reward of the target MDP $M$.
The baseline threshold $\tau(s)$ represents the minimally acceptable performance for the source policy $\mu$ to be used under the target MDP $M$.
Given the source policy $\mu(a|s)$ and initiation set $L$, we may now define the primal option.  

\begin{defn}[Primal Option]
For a given MDP $M$, a primal option is the tuple $(\mu, L, \beta)$, where $\mu$ is a Markov policy, $L$ is the initiation set in Definition \ref{def_initiation_set} and $\beta$ is the termination function.
The termination function is defined by
\begin{equation}
    \beta(s)=
    \begin{cases}
        1 & \text{ if $s \notin L$ or $s$ is a terminal state in $M$}\\
        0 & \text{ otherwise.}
    \end{cases}
    \label{eq_source_termination}
\end{equation}
\label{def_primal_option}
\end{defn}

\subsection{Between Primal and Learner Options}
The initiation set $L$ of the primal option delineates the scope of success for the source policy $\mu$ in the target MDP $M$.
Naturally, we expect the primal option to fail to meet our performance threshold $\tau(s)$ in parts of the state space, i.e. $S \setminus L$.
Given that the policy $\mu$ under the primal option is fixed, we will need other parts of the architecture to learn to handle at least $s \in S \setminus L$, if not also to improve over the primal option for $s \in L$.
For that purpose, we use another option called the \textit{learner option}.
To choose between the primal and the learner options, we introduce a \textit{main policy} that operates at a higher level:

\begin{defn}[Main Policy]
Given the target MDP $M$, the primal option $o^P$, and the learner option $o^L$, the main policy is a deterministic policy $\pi^{main}(s)$ that chooses between $o^P$ and $o^L$ given the current state: 
\begin{equation}
    \pi^{main}(s)=
    \begin{cases}
        o^P & \text{ if $s \in L$}\\
        o^L & \text{ otherwise}
    \end{cases}
    \label{eq_main_policy}
\end{equation}
\label{def_main_policy}
\end{defn}

Two points are worth noting about this definition, because they further illustrate the central importance for LISPR of the learning of the initiation set $L$.
First, the learned initiation set $L$ is used to determine the switching.
This should not be surprising given that $L$ is defined by comparing the evaluated performance of $\mu$ in $M$ with the performance criterion $\tau(s)$.
Second, while the main policy is a heuristic, it is still learned from, or at least adaptive with respect to, the target MDP $M$, because $L$ is learned through policy evaluation of $\mu$ on $M$. 

\subsection{Learning to Recover}
\label{sec_learning_to_recover}
Depending on the specific requirements on the transfer learning we do, the learner option $o^L$ in Definition \ref{def_main_policy} may have one of two goals: (1) transition to a state $s \in L$ where the primal option will successfully complete the task, or (2) directly maximize the reward of the target MDP $M$.
We will call the first variety of learner option a \textit{recovery option} and the second a \textit{student option}.
In this subsection and the next, we present the theory for learning them and make claims about their theoretical properties that are proved in the appendix.

\begin{defn}[Recovery Option]
Given an episodic MDP $M$, a recovery option $o^R$ is the tuple $(\pi^{R}(a|s), \mathcal{I}, \beta(s), o^P)$ where $\pi^{R}$ is the recovery policy, $\mathcal{I}$ is the initiation set, $\beta(s)$ is the termination function, and $o^P$ is a primal option according to Definition \ref{def_primal_option}, such that the initiation set $\mathcal{I}=S \setminus L$.
The recovery policy maximizes the value function
\begin{equation}
    \begin{split}
        V^{R}(s_t) & =\E_{\pi^{R}} \left[ \sum_{i=0}^{\infty}{ \left( \prod_{j=0}^{i-1}{\gamma(1-\beta(s_{t+j+1}))} \right) r^{R}_{t+i+1}} \right]
    \end{split}
    \label{eq_recovery_value}
\end{equation}
where $L$ is the initiation set of the primal option $o^P$, the termination probability $\beta(s)$ is given as 
\begin{equation}
    \beta(s)=
        \begin{cases}
            1 & \text{if $s \in L$}\\
            0 & \text{otherwise,}
        \end{cases}
    \label{eq_recovery_termination}
\end{equation}
and the reward $r^{R}_{t+1}$ is specified in terms of target MDP reward $r_{t+1}$ and successor predictor $G(s_{t+1})$:
\begin{equation}
    r^{R}_{t+1}=
    \begin{cases}
        r_{t+1} + \gamma G(s_{t+1}) & \text{if $s_{t+1} \in L$} \\
        r_{t+1} & \text{otherwise.}
    \end{cases}
    \label{eq_recovery_reward}
\end{equation}
\label{def_recovery_option}
\end{defn}

The policy in Definition \ref{def_recovery_option} is called an ``recovery'' policy because it learns to hand over control to the primal option via the main policy at the earliest possible time.
This is because according to Equation \ref{eq_recovery_reward}, the terminal reward includes $G(s)$ and thus the agent is rewarded for directly entering states with performance greater than the threshold $\tau(s)$.
Under such a setup, if $\tau(s)$ is chosen to be $V^{R}(s)$, the main policy in Definition \ref{def_main_policy} is optimal:

\begin{restatable}[Recovery Optimality]{thm}{optimalrecoverytheorem}
Given primal option $o^P$ with $\tau(s)=V^{R}(s)$ and recovery option $o^R$, the main policy $\pi^{main}(s)$ is optimal in that $V^{main}(s)=V^{*}(s)$ where $V^{*}(s)=\max_{\forall \pi}{V^{\pi}(s)}$ when the recovery policy $\pi^{R*}(a|s)$ of $o^{R*}$ is optimal in that $\pi^{R*} = \argmax_{\forall \pi}{\E_{\pi} \left[ \sum_{i=0}^{\infty}{\gamma^i r^{R}_{t+i+1}} \right]}$.
\label{thm_recovery_optimality}
\end{restatable}

The importance of the theorem is that, the result holds even when the primal option is sub-optimal.
Furthermore, there is a lower bound on the performance of the main policy with a recovery option:

\begin{restatable}[Recovery Lower Bound]{lem}{boundrecoverylemma}
Given primal option $o^P$ with $\tau(s)=V^{R}(s)$ and recovery option $o^R$, the value of the main policy $V^{\pi}(s)$ satisfies both of the following lower bounds $V^{\pi}(s) \geq V^{R}(s)$ and $V^{\pi}(s) \geq G(s)$ for all $s \in S$.
\label{lem_recovery_lower_bound}
\end{restatable}

Using Lemma \ref{lem_recovery_lower_bound}, we can guarantee improvement in the true value of the main policy when the true value of the recovery policy improves.
This leads to the next theorem.

\begin{restatable}[Improvement with Recovery Option]{thm}{improvementrecoverytheorem}
Given a primal option $o^P$ with $\tau(s)=V^{R}(s)$ and a recovery option $o^R$, if the recovery policy $\pi^{R}(a|s)$ is updated to $\pi^{R'}(a|s)$ such that $V^{R'}(s) \geq V^{R}(s)$ for all states $s \in S$, the value of the new main policy $\pi'(s)$ is guaranteed to improve as well, i.e. $V^{\pi'}(s) \geq V^{\pi}(s)$ for all states $s \in S$.
\label{thm_recovery_improvement}
\end{restatable}

Theorem \ref{thm_recovery_improvement} means that improvements to the recovery policy will always provide improvements in the overall performance of the agent under our framework when $\tau(s)=V^{R}(s)$.

\subsection{Learning to Master}
\label{sec_learning_to_master}
In contrast to the recovery option, the student option maximizes the reward of the target MDP using all the samples collected by both the primal option and itself.
Thus, off-policy methods are typically used to learn the student option.

\begin{defn}[Student Option]
Given the target MDP $M$, a student option $o^S$ is the tuple $(\pi^{S}(a|s), \mathcal{I}, \beta(s))$, where the student policy is $\pi^{S}$, the initiation set $\mathcal{I}$ is the set of all states $S$, and the termination function $\beta(s)=0$ for all states $s \in S$ except for terminal states in $M$ if it is episodic, such that the policy directly maximizes the value function
\begin{equation}
    V^{S}(s)=\E_{\pi^{S}} \left[ \sum_{i=0}^{\infty}{ \gamma^i r_{t+i+1}} \right]
    \label{eq_student_value}
\end{equation}
where $r_{t+i+1}$ is the reward of $M$.
\label{def_student_option}
\end{defn}

The student option uses the primal option as a kind of tutor in order to improve learning where the source policy is optimal but the student policy is sub-optimal.

\begin{restatable}[Student Optimality]{thm}{optimalstudenttheorem}
Given primal option $o^P$ with $\tau(s)=V^{S}(s)$ and student option $o^S$, the main policy $\pi^{main}(s)$ is optimal in that $V^{main}(s)=V^{*}(s)$ where $V^{*}(s)=\max_{\forall \pi}{V^{\pi}(s)}$ when the student policy $\pi^{S*}(a|s)$ of $o^{S*}$ is optimal in that $\pi^{S*} = \argmax_{\forall \pi}{\E_{\pi} \left[ \sum_{i=0}^{\infty}{\gamma^i r_{t+i+1}} \right]}$.
\label{thm_student_optimality}
\end{restatable}

There is also a lower bound to the performance on the main policy using a student option:

\begin{restatable}[Student Lower Bound]{lem}{boundstudentlemma}
Given primal option $o^P$ with $\tau(s)=V^{S}(s)$ and student option $o^S$, the value of the main policy $V^{\pi}(s)$ satisfies both of the following lower bounds $V^{\pi}(s) \geq V^{S}(s)$ and $V^{\pi}(s) \geq G(s)$ for all $s \in S$.
\label{lem_student_lower_bound}
\end{restatable}

Using Lemma \ref{lem_student_lower_bound}, we can guarantee improvement in the true value of the main policy as the true value of the student policy improves. This leads to the next theorem.

\begin{restatable}[Improvement Theorem with Student Options]{thm}{improvementstudenttheorem}
Given any MDP $M$, primal option $o^P$ with $\tau(s)=V^{S}(s)$ and student option $o^S$, if the student policy $\pi^{S}(a|s)$ is updated to $\pi^{S'}(a|s)$ such that $V^{S'}(s) \geq V^{S}(s)$ for all states $s \in S$, the value of the new main policy $\pi'(s)$ is guaranteed to improve as well such that $V^{\pi'}(s) \geq V^{\pi}(s)$ for all states $s \in S$.
\label{thm_student_improvement}
\end{restatable}

In light of the parallel theoretical results above (Theorem \ref{thm_recovery_optimality}, Lemma \ref{lem_recovery_lower_bound}, Theorem \ref{thm_recovery_improvement} and Theorem \ref{thm_student_optimality}, Lemma \ref{lem_student_lower_bound}, Theorem \ref{thm_student_improvement}), both recovery option and student option lead to the same result in the LISPR framework: optimality of the main policy is guaranteed if learning of recovery and student policies achieves optimality respectively.
While these two approaches are equivalent in this abstract sense, an important difference lies in how transfer learning actually happens.
After the source policy is packaged as the primal option, the recovery option learns to defer to the primal option wherever the latter is already optimal, whereas the student option learns also from experiences collected through the primal option to cover the whole state space.

\subsection{Practical Considerations}

While the LISPR framework is architecturally clean and theoretically well-grounded and a detailed algorithm is available in the appendix, several things should be mentioned that are important for a successful implementation: (1) convergence of the recovery policy since its reward depends on estimates of the source policy performance, (2) the effect of overestimation or maximisation bias when learning the value function of the recovery and student policies, and (3) exploration with the main policy.

Regarding convergence, no theoretical results are provided however it should be noted that when when learning $G(s)$ and $\pi^{R}(a|s)$ at the same time, the replay buffer may contain outdated rewards that may affect convergence or the rate of convergence.
One solution is to recompute the terminal rewards when learning a recovery policy with every minibatch sample.

The overestimation bias results in the value functions $V^{R}(s)$ and $V^{S}(s)$ for the recovery policy and student policy respectively to be over-estimated in practice during learning: a common challenge with DQN and actor-critic methods \citep{fujimoto2018maxbias}. 
The consequence of this in the LISPR framework is that the primal option will be under-utilized during training which may impact learning.
Thus a better approximation of $V^{R}(s)$ and $V^{S}(s)$ is needed for practical implementations of the LISPR framework.
What we did in our implementation that worked surprisingly well was to learn an on-policy value function of the main policy $V^{\pi}(s)$ using standard TD learning such that $\tau(s)=V^{\pi}(s)$.
This is inherently biased when trained with a replay buffer but this did not seem to have a substantial affect on performance when training recovery and student policies.

To ensure exploration, we used $\epsilon$-greedy exploration like many other policy transfer and policy-reuse approaches \citep{xie2018trainingwheels, barreto2017successorfeatures, fernando2010probpolicyreuse, siyuan2019contextawarepolicyreuse, yang2020ptf}.
Thus, with probability $\epsilon \in [0,1]$ at each time step, one of the primal option $o^P$ and the learner option $o^L$ is selected, otherwise the main policy is followed as per Definition \ref{def_main_policy}.
However, more sophisticated exploration mechanisms our needed that commit to the execution of options over longer time horizon.
One such approach that deserves more investigation in future work is the Ornstein-Ulhembeck exploration commonly used in DDPG \citep{lillicrap2016ddpg}.
Since the threshold $\tau(s)$ in Definition \ref{def_initiation_set} is a real-value, OU exploration can be applied to $\tau(s)$ to achieve exploration of the options over multiple time steps.
However, effective methods of exploration of the main policy is beyond the scope of this paper.
\section{Experiments}

We demonstrate the LISPR framework in three environments:  (1) two tabular worlds called "multiroom world" and "box world", (2) a navigation task with high-dimensional 2D lidar observations called the "social navigation task", (3) the "lunar lander task", and (4) the "bipedal walker task".
All results, unless reported otherwise, are test results rather than training results where training is interrupted every 1000 iterations to report the average return over 10 random instances of the environment.
Experiments are repeated 10 times where a 95\% confidence interval is shaded around the mean return to highlight the confidence bounds.
The lunar lander results are provided in the Appendix.

\subsection{Tabular Examples}
\subsubsection{Multiroom World}
In order to demonstrate and visualize the LISPR framework, we look to tabular domains.
Several tasks were constructed using pycolab.
The first task involved transfer from a two room grid world to a four room grid world depicted in Figure \ref{fig_tabular_tasks}.
The value functions are not transferred to the target domain; instead, the values for $G(s)$ and the recovery policy are all set to zero and must be learned from scratch; only the policy is transferred through a simple state transformation between the grid worlds.
Each state in the source grid world is represented by a 2D coordinate that is mapped to a state in the target grid world to enable transfer.
The agent starts in any open space and receives a terminal reward of +1 for reaching the fixed goal location.
The value estimates for the source task and target task are visualized in Figure \ref{fig_tabular_tasks}.

\begin{figure}
    \centering
    \includegraphics[width=0.45\textwidth]{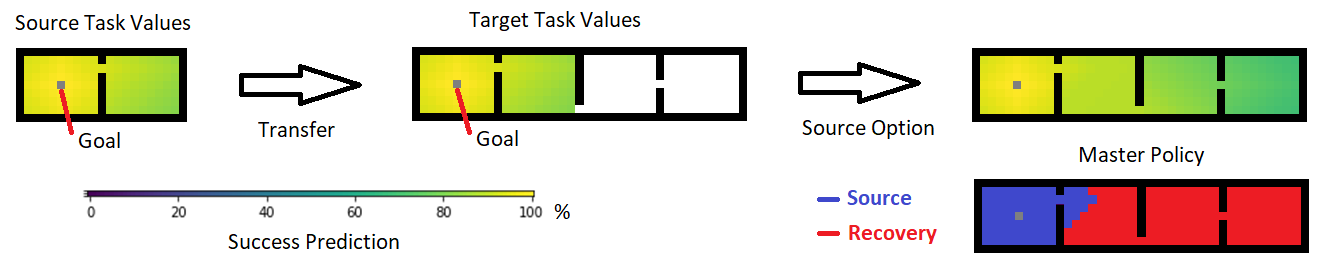}
    \caption{Visualisation of the value functions in the source and target tasks in \textbf{multiroom world} where yellow predicts higher success than green and white predicts zero success for $\gamma=0.99$.  The main policy decision is depicted for a threshold of $\tau(s)=0.9$.}
    \label{fig_tabular_tasks}
\end{figure}

\subsubsection{Box World}
The \textbf{multiroom world} task is quite simple for illustration of the idea.
Q-learning with eligibility traces trains very quickly in the target task and thus very little improvement in learning speed is expected.
To understand the effectiveness of the LISPR framework, a more complex grid world task was constructed with the goal of pushing a box onto a goal square. 
The source task is a single room and the target task is a larger room where the state is mapped according to the fixed goal position.
The reward is $+1$ when the box is pushed onto the goal square.
The source grid world only required about 100K samples to train using Q-Learning with eligibility traces.
The target grid world required over 1M samples.
A noticeable improvement in learning efficiency is demonstrated when learning with the source option from the source task.

\begin{figure}
    \centering
    \includegraphics[width=0.45\textwidth]{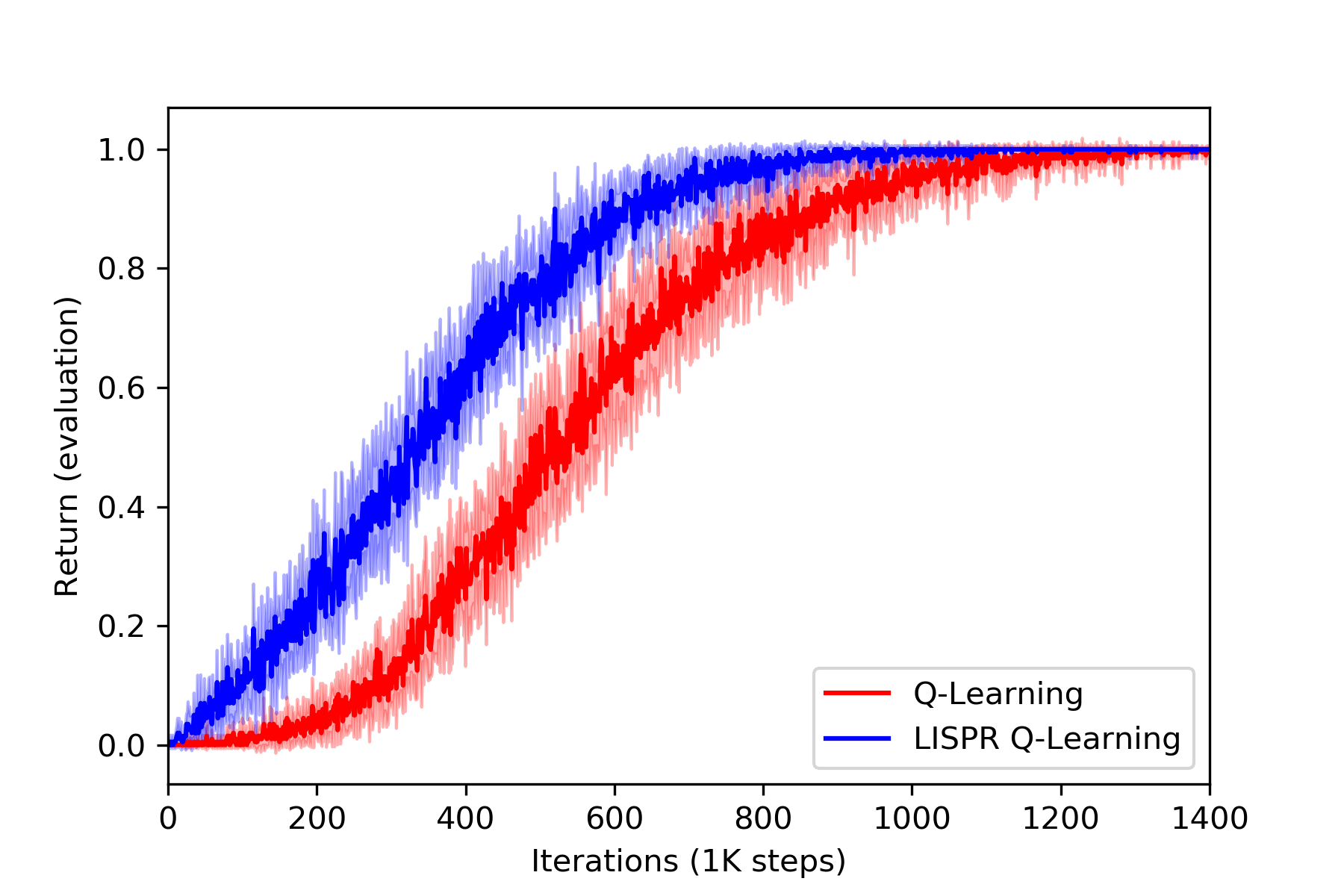}
    \caption{Demonstration of sample efficiency improvement with the LISPR framework in \textbf{box world}}
\label{fig_tabular_performance}
\end{figure}

Suitable parameter sweeps were conducted to find the best $\epsilon$-greedy value for the main policy, along with learning rates and $\lambda$ for the eligibility trace.
Details are given in the Appendix.

\subsection{Social Navigation with a Heuristic Source Policy}
We demonstrate LISPR on a substantially more challenging task of navigating a robot through a small group of moving pedestrians modelled with the ORCA \citep{berg2011orca} collision avoidance algorithm.
The robot receives high dimensional 2D laser scans of the environment and the position of the goal location.
The simulator was developed in-house and is the same simulator used in \citep{jinj2020socialnavigation} with a few changes:  (1) the environment includes about two to three pedestrians, (2) the control frequent is 20Hz, and (3) the observation that includes a history of twenty 2D laser scans.
The changes were designed to reduce the learning time needed to train the agent while still retaining the complexity of navigating in a dynamic environment.
The reward structure was defined to be sparse such that the agent receives $+1$ for reaching the destination, $-1$ for collision and zero otherwise.
Agents are given 1 million samples of experience which results in 500K updates (iterations) to the models; an update is applied every two steps in the environment.
More details are found in the appendix.

The source policy is a very simple heuristic policy based on the relative position of the goal to determine the differential steer velocities for the left and right wheels.
The heuristic policy receives a return of just under 0.1 on average due to frequent collisions.
The source policy performs quite poorly on the task since it only "sees" the goal; nevertheless, the LISPR framework can exploit the source policy to quickly learn a new policy that avoids pedestrians.

\begin{figure}
    \centering
    \includegraphics[width=0.45\textwidth]{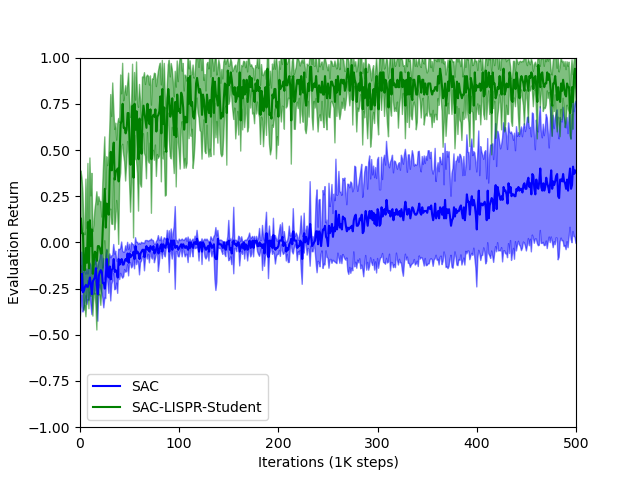}
    \caption{Comparison of performance on \textbf{social navigation task} with soft-actor critic (SAC) and LISPR  with a student policy learned via SAC (SAC-LISPR-Student)}
    \label{fig_toy_sac}
\end{figure}

Figure \ref{fig_toy_sac} shows that the LISPR solution is reached within 100K updates to the model (iterations) but when learned tabular rasa, SAC learns within 1M updates to the model (iterations).
Both "recovery" and "student" policies effectively have the same learning curves (not shown).
With the right exploration policy however the performance gap narrows significantly indicating that a big advantage of the LISPR framework is through improved exploration.
The Ornstein-Ulhenbeck exploration process of deep deterministic policy gradient (DDPG), when properly tuned, is nearly as effective as LISPR in learning a suitable policy; more details are in the Appendix.
The neural network architectures for both DDPG and SAC were the same as \citep{jinj2020socialnavigation}.

\subsection{Bipedal with a Neural Source Policy}
LISPR was also applied to the BipedalWalker Hardcore environment in gym.
The transfer task is defined by learning a policy in BipedalWalker and transferring to BipedalWalker Hardcore:  we call this the \textbf{bipedal walker task}.
We started by training two policies with DDPG and SAC respectively in BipedalWalker that were evaluated to reach satisfactory performance when tested.
These policies were transferred to BipedalWalker Hardcore as source policies using the LISPR framework.
Student and recovery policies were trained for 2M updates to the models which equated to 4M samples in the environment since an update was completed every two steps.
Comparisons to progressive networks \citep{rusu2017} and warm-start initialization were done with the same policies from the source environment.
After significant hyperparameter tuning, it was found that frame skipping with three frames resulted in a significant performance improvement.
The DDPG exploration used the Ornstein-Uhlenbeck process that was fine tuned to achieve the best possible performance in BipedalWalker.
Unfortunately, training DDPG in BipedalWalker Hardcore was quite challenging and inconsistent,  whereas training with SAC was more consistent and performed better.
Only a few of the runs of DDPG were able to achieve very good final performance after 2M updates to the model (iterations) including progressive networks; on average, SAC baselines performed reasonably well.
Warm initialization showed very minor improvement with SAC as shown in Figures \ref{fig_bipedal_ddpg} and \ref{fig_bipedal_sac}.

\begin{figure}
    \centering
    \includegraphics[width=0.45\textwidth]{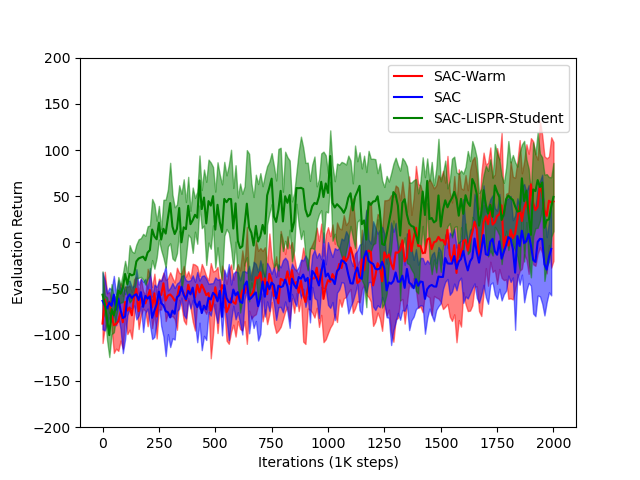}
    \caption{Comparison of performance on \textbf{bipedal walker task} with SAC, warm-start initialization with SAC, and LISPR with a student policy learned via SAC (SAC-LISPR-Student)}
    \label{fig_bipedal_sac}
\end{figure}
\begin{figure}
    \centering
    \includegraphics[width=0.45\textwidth]{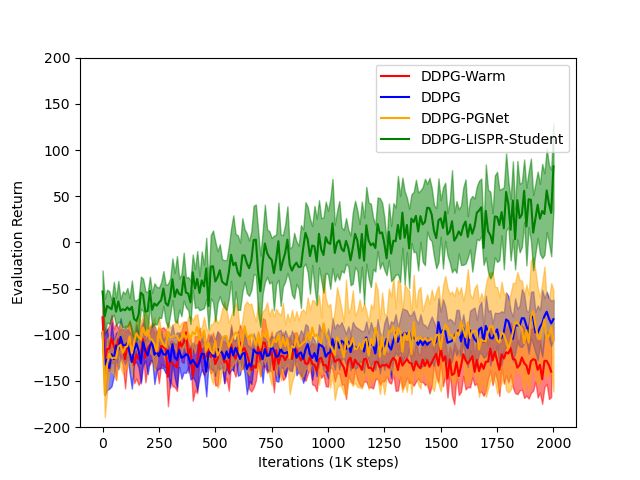}
    \caption{Comparison of performance on \textbf{bipedal walker task} with DDPG, warm-start initialization with DDPG, progressive networks with DDPG and LISPR  with a student policy learned via DDPG (DDPG-LISPR-Student)}
    \label{fig_bipedal_ddpg}
\end{figure}

One hypothesis for the difficulty in transfer between these tasks is the covariate shift in the lidar observations of the bipedal walker between the source and target environments due to the obstacles, pits and stairs in the target environment that never appear in the source environment.
The covariate shift may result in a poor initialization and poor transfer of features.
The LISPR framework, on the other hand, learns when the policy is "good" and "bad" and only uses the policy when it is predicted to be good.
Thus, anomalies as a result of the covariate shift are predicted and handled by the recovery policy.
We also found that there is no significant difference between the performance of the recovery policy and student policy.
\section{Conclusions}
The LISPR framework is a specialization of the options framework that provides a principled approach to learning the initiation set of a policy using general value functions.
It does not assume a particular implementation of the source policy.
This allows for an interesting blend of heuristic and learned policies without sacrificing optimality.
The framework has some nice theoretical properties including an optimal definition for the main policy, guarantee in improvement if the true value of the learned policy improves, and guarantee of minimum performance.

Two specific choices of learning architecture, learning to master and learning to recover, are presented and shown to be equivalent under certain conditions, with the first easier to implement and the second possibly indispensable in practice either because exploration in the part of the state space covered by the source policy is impossible or because the source policy has to be integrated as is for other reasons such as certification or license restriction.

Experiments have demonstrated successful transfers wherein the source policies were reused to solve a harder problem in a way similar to what happens in the zone of proximal development \citep{vygotsky1978}.
It is worth noting that some source options are very helpful in solving sparse reward problems.
Additionally, our experiments suggest that the LISPR framework is less sensitive to covariate shift that happens during policy transfer with reinforcement learning.
On the other hand, performance improvements may be minimal on tasks with extensive reward shaping, suggesting that much of the performance gains are due to better exploration.

Some unanswered questions include a possible theoretical relationship between LISPR and reward shaping and improving exploration of the main policy.
It is also natural to extend the framework to learning goal-conditioned initiation sets as in \citep{demir2019} by making use of, for example, universal value functions \citep{schaul2015uvfa}.
Finally, while we designed the LISPR framework with an eye on real-world robotics applications, its benefits there remain to be demonstrated in future work, and we suspect that further system-level constraints and requirements will need to be incorporated in both learning of the initiation set and learning to recover or master.
 
\bibliographystyle{apalike}
\bibliography{references}

\clearpage
\newpage
\appendix
\section{Appendix}
\subsection{Recovery Option Proofs}

In this section, we prove the theorems and lemmas concerning the learning of recovery options as defined in Definition \ref{def_recovery_option}.  We start with the proof of optimality which is straightforward if we assume that recovery policy learning is optimal.  While such an assumption is strong, the point is that the LISPR framework guarantees optimality for the overall main policy under such an assumption.

\optimalrecoverytheorem*
\begin{proof}
Let us assume that the recovery policy $\pi^{R*}(a|s)$ is optimal such that the value $V^{R*}(s)=\max_{\forall \pi}{V^{\pi}(s)}$ for all states $s \in S$.  We define the threshold $\tau(s)=V^{R*}(s)$ for all states $s \in S$.  Let us consider the following cases:
\setcounter{casep}{0}
\begin{casep}[$s_t \in L$]
\label{recovery_optimality_case_s_in_L}
From Definition \ref{def_initiation_set} we have the condition that $G(s_t) \geq V^{R*}(s_t)$ for all states $s_t \in L$.  However, due to $\pi^{R*}(a|s)$ being optimal, we also have $G(s_t) \leq V^{R*}(s_t)$.  Therefore, $G(s_t)=V^{R*}(s_t)$ for all states $s_t \in L$.  This means that the primal option is optimal if $s_t \in L$.  Furthermore, the value of the main policy is given by
\begin{equation}
    \begin{split}
        V^{\pi}(s_t)  & = \E_{\mu} \left[ \sum_{i=0}^{\infty}{\gamma^i r_{t+i+1}} \right]\\
                    & = G(s_t)
    \end{split}
\end{equation}
\end{casep}
\begin{casep}[$s_k \notin L$ for all $k \geq t$]
From the main policy definition, we have $\pi(s_t)=o^R$ for all states $s_t \notin L$.  Therefore, the value of the main policy is given by 
\begin{equation}
    \begin{split}
        V^{\pi}(s_t)  & = \E_{\pi^{R*}} \left[ \sum_{i=0}^{\infty}{\gamma^i r_{t+i+1}} \right]\\
                    & = V^{R*}(s_t)
    \end{split}
\end{equation}
\end{casep}
\begin{casep}[$s_k \notin L$ for all $t \leq k < T$ where $s_T \in L$]
Let us consider a future state $s_T$ for $T>t$ where $s_T \in L$ such that all states $s_k \notin L$ from $k=t ... T-1$, then $V^\pi(s_t)$ can be written as
\begin{equation}
    \begin{split}
        V^{\pi}(s_t)  & = \E_{\pi} \left[ \sum_{i=0}^{\infty}{\gamma^i r_{t+i+1}} \right]\\
                    & = \E_{\pi^{R*}} \left[ \sum_{i=0}^{T-t-1}{\gamma^i r_{t+i+1}} \right] + \gamma^{T-t} V^{\pi}(s_T)\\
                    & = V^{R*}(s_t) - \gamma^{T-t} V^{R*}(s_T) + \gamma^{T-t} V^{\pi}(s_T)\\
                    & = V^{R*}(s_t) + \gamma^{T-t} \left( V^{\pi}(s_T) - V^{R*}(s_T) \right)
    \end{split}
    \label{eq_main_value_s_notin_L}
\end{equation}

Since $s_T \in L$, we can use \cref{recovery_optimality_case_s_in_L} to get 
\begin{equation}
    \begin{split}
        V^{\pi}(s_t) & = V^{R*}(s_t) + \gamma^{T-t} \left( G(s_T) - G(s_T) \right)\\
                    & = V^{R*}(s_t)
    \end{split}
\end{equation}
\end{casep}

Taking the three cases together, we have $V^{\pi}(s)=V^{R*}(s_t)$ for all states $s_t \in S$.  Therefore, $\pi(s)$ in Definition \ref{def_main_policy} is optimal if $\pi^{R*}(a|s)$ is optimal.

\end{proof}

Next, we introduce some notations for finite horizon values that will be useful in future proofs.

\begin{defn}[Finite Horizon for Primal Option]
The finite horizon value for trajectory of states $s_t ... s_{T-1}$ under the primal option is given by 
\begin{equation}
    \begin{split}
        H_T(s_t) & = \E_{\mu} \left[ \sum_{i=0}^{T-1}{\gamma^i r_{t+i+1}} \right]\\
            & = G(s_t) - \gamma^{T-t} G(s_T)
    \end{split}
    \label{eq_finite_horizon_source}
\end{equation}
For ease of notation, we define $H_t(s_t)=0$.
\label{def_finite_horizon_source}
\end{defn}

\begin{defn}[Finite Horizon for Recovery Option]
The finite horizon value for trajectory of states $s_t ... s_{T-1}$ under the recovery option is given by 
\begin{equation}
    \begin{split}
        H_T^{R}(s_t) & = \E_{R} \left[ \sum_{i=0}^{T-1}{\gamma^i r_{t+i+1}} \right]\\
            & = V^{R}(s_t) - \gamma^{T-t} V^{R}(s_T)
    \end{split}
    \label{eq_finite_horizon_recovery}
\end{equation}
For ease of notation, we define $H^{R}_t(s_t)=0$.
\label{def_finite_horizon_recovery}
\end{defn}

To prove the improvement Theorem \ref{thm_recovery_improvement}, we need Lemmas \ref{lem_finite_horizon_recovery_inequality_s_in_L} and \ref{lem_finite_horizon_recovery_inequality_s_notin_L} below.

\begin{lem}[Finite Horizon Inequality Starting With the Primal Option]
If there is a trajectory of states from $s_t \in L$ to $s_T \notin L$ such that $G(s_T) < V^R(s_T)$, where $s_k \in L$ for $k=t...T-1$, then $H_T(s_t) \geq H_T^{R}(s_t)$.
\begin{proof}
\begin{equation}
    \begin{split}
        G(s_t) & \geq V^{R}(s_t)\\
        G(s_t) - \gamma^{T-t} G(s_T) & \geq V^{R}(s_t) - \gamma^{T-t} V^{R}(s_T)\\
        H_T(s_t) & \geq H_T^{R}(s_t)
    \end{split}
    \label{eq_finite_horizon_recovery_inequality_s_in_L}
\end{equation}
\end{proof}
\label{lem_finite_horizon_recovery_inequality_s_in_L}
\end{lem}

\begin{lem}[Finite Horizon Inequality Starting with the Recovery Policy]
If there is a trajectory of states from $s_t \notin L$ to $s_T \in L$, where $s_k \notin L$ for $k=t...T-1$, then $H_T(s_t) < H_T^{R}(s_t)$.
\begin{proof}
\begin{equation}
    \begin{split}
        G(s_t) & < V^{R}(s_t)\\
        G(s_t) - \gamma^{T-t} G(s_T) & < V^{R}(s_t) - \gamma^{T-t} V^{R}(s_T)\\
        H_T(s_t) & < H_T^{R}(s_t)
    \end{split}
    \label{eq_finite_horizon_recovery_inequality_s_notin_L}
\end{equation}
\end{proof}
\label{lem_finite_horizon_recovery_inequality_s_notin_L}
\end{lem}

We will now prove Lemma \ref{lem_recovery_lower_bound}, which says that the performance of the recovery policy has a lower bound.

\boundrecoverylemma*
\begin{proof}
Consider the following cases given any state $s_t \in S$:
\setcounter{casep}{0}
\begin{casep}[Trajectory starting in state $s_t \in L$]
Let us denote times where the main policy changes between ``source'' and ``recovery'' as $T_{j+1} > T_{j} > t$ for $j=1 ... N$ where $N$ is the number of total changes starting from state $s_t$.  We can write the value of the main policy in terms of finite horizons given in Definitions \ref{def_finite_horizon_source} and \ref{def_finite_horizon_recovery}.  Given Lemmas \ref{lem_finite_horizon_recovery_inequality_s_in_L} and \ref{lem_finite_horizon_recovery_inequality_s_notin_L}, the main policy value function can be written as
\begin{equation}
    \begin{split}
        V^{\pi}(s_t) & = H_{T_1}(s_t) + \gamma^{T_1-t} H_{T_2}^{R}(s_{T_1}) + \gamma^{T_2-t} H_{T_3}(s_{T_2}) + ...\\
                & \geq H_{T_1}(s_t) + \gamma^{T_1-t} H_{T_2}(s_{T_1}) + \gamma^{T_2-t} H_{T_3}(s_{T_2}) + ...\\
                & \geq G(s_t)
    \end{split}
    \label{eq_vpi_finite_horizon_recovery_st_in_L}
\end{equation}
Since $s_t \in L$, we also have $G(s_t) \geq V^{R}(s_t)$ by Definition \ref{def_initiation_set}.  Therefore, $V^{\pi}(s_t) \geq G(s_t) \geq V^{R}(s_t)$ for any trajectory starting in state $s_t \in L$.
\end{casep}
\begin{casep}[Trajectory starting in state $s_t \notin L$]
Let us denote times where the main policy changes between ``source'' and ``recovery'' as $T_{j+1} > T_{j} > t$ for $j=1 ... N$ where $N$ is the number of total changes starting from state $s_t$.  We can write the value of the main policy in terms of finite horizons given in Definitions \ref{def_finite_horizon_source} and \ref{def_finite_horizon_recovery}.  Given Lemmas \ref{lem_finite_horizon_recovery_inequality_s_in_L} and \ref{lem_finite_horizon_recovery_inequality_s_notin_L}, the main policy value function can be written as
\begin{equation}
    \begin{split}
        V^{\pi}(s_t) & = H_{T_1}^{R}(s_t) + \gamma^{T_1-t} H_{T_2}(s_{T_1}) + \gamma^{T_2-t} H_{T_3}^{R}(s_{T_2}) + ...\\
                & \geq H_{T_1}^{R}(s_t) + \gamma^{T_1-t} H_{T_2}^{R}(s_{T_1}) + \gamma^{T_2-t} H_{T_3}^{R}(s_{T_2}) + ...\\
                & \geq V^{R}(s_t)
    \end{split}
    \label{eq_vpi_finite_horizon_recovery_st_not_in_L}
\end{equation}
Since $s_t \notin L$, we also have $G(s_t) < V^{R}(s_t)$ by Definition \ref{def_initiation_set}.  Therefore, $V^{\pi}(s_t) \geq V^{R}(s_t) > G(s_t)$ for any trajectory starting in state $s_t \notin L$.
\end{casep}
Combining these cases, we conclude that $V^{\pi}(s) \geq V^{R}(s)$ and $V^{\pi}(s) \geq G(s)$ for all $s \in S$.
\end{proof}

We go one step further to define the value of the main policy according to $G(s)$ and $V^{R}(s)$.

\begin{lem}[Main Policy Value with Recovery Option]
The value of the main policy $V^{\pi}(s)$ is equal to $V^{\pi}(s)+V^{\bar{\pi}}(s)=G(s)+V^{R}(s)$ given recovery policy $\pi^{R}(a|s)$ where $\bar{\pi}(s)$ is the ``anti-main'' policy where the decision rule in Equation \ref{eq_main_policy} is reversed such that $o^P$ is chosen if $s \notin L$ and $o^R$ otherwise.
\begin{proof}
Let us consider the following cases:
\setcounter{casep}{0}
\begin{casep}[State $s \in L$]
We can write $V^{\pi}(s_t) - V^{R}(s_t)$ as a function of finite horizons:
\begin{equation}
    \begin{split}
        V^{\pi}(s_t) - V^{R}(s_t) = & H_{T_1}(s_t) + \gamma^{T_2-t} H_{T_3}(s_{T_2}) + ...\\
            & - (H_{T_1}^{R}(s_t) + \gamma^{T_2-t}H_{T_3}^{R}(s_{T_2}) + ...)
    \end{split}    
\end{equation}
where $t \leq T_1 \leq T_2 \leq ...$ such that $s_t \in L, s_{T_1} \notin L, s_{T_2} \in L, ...$.  Similarly we can write $V^{\pi}(s_t) - G(s_t)$ as a function of finite horizons:
\begin{equation}
    \begin{split}
        V^{\pi}(s_t) - G(s_t) = & \gamma^{T_1-t} H_{T_2}^{R}(s_{T_1}) + \gamma^{T_3-t} H_{T_4}^{R}(s_{T_3}) + ...\\
            & - (\gamma^{T_1-t} H_{T_2}(s_{T_1}) + \gamma^{T_3-t} H_{T_4}(s_{T_3}) + ...)
    \end{split}
\end{equation}
Adding $V^{\pi}(s_t) - V^{R}(s_t)$ and $V^{\pi}(s_t) - G(s_t)$ together, gives us 
\begin{equation}
    \begin{split}
        2V^{\pi}(s_t) - G(s_t) - V^{R}(s_t) & = V^{\pi}(s_t) - V^{\bar{\pi}}(s_t)\\
        V^{\pi}(s_t) + V^{\bar{\pi}}(s_t) & = G(s_t) + V^{R}(s_t)
    \end{split}
\end{equation}
\end{casep}
\begin{casep}[State $s \notin L$]
Similar to the previous case, we can write $V^{\pi}(s_t) - V^{R}(s_t)$ as a function of finite horizons:
\begin{equation}
    \begin{split}
        V^{\pi}(s_t) - V^{R}(s_t) = & \gamma ^{T_1-t}H_{T_2}(s_{T_1}) + \gamma^{T_3-t} H_{T_4}(s_{T_3}) + ...\\
            & - (\gamma ^{T_1-t}H_{T_2}^{R}(s_{T_1}) + \gamma^{T_3-t} H^{R}_{T_4}(s_{T_3}) + ...)
    \end{split}
\end{equation}
where $t \leq T_1 \leq T_2 \leq ...$ such that $s_t \notin L, s_{T_1} \in L, s_{T_2} \notin L, ...$.  Similar we can write $V^{\pi}(s_t) - G(s_t)$ as a function of finite horizons of the recovery policy:
\begin{equation}
    \begin{split}
        V^{\pi}(s_t) - G(s_t) = & H_{T_1}^{R}(s_t) + \gamma^{T_2-t}H_{T_3}^{R}(s_{T_2}) + ...\\
            & - (H_{T_1}(s_t) + \gamma^{T_2-t}H_{T_3}(s_{T_2}) + ...)
    \end{split}
\end{equation}
Adding $V^{\pi}(s_t) - V^{R}(s_t)$ and $V^{\pi}(s_t) - G(s_t)$ together, gives us 
\begin{equation}
    \begin{split}
        2V^{\pi}(s_t) - G(s_t) - V^{R}(s_t) & = V^{\pi}(s_t) - V^{\bar{\pi}}(s_t)\\
        V^{\pi}(s_t) + V^{\bar{\pi}}(s_t) & = G(s_t) + V^{R}(s_t)
    \end{split}
\end{equation}
\end{casep}
Therefore, for all states $s \in S$, we have $V^{\pi}(s)+V^{\bar{\pi}}(s)=G(s)+V^{R}(s)$.
\end{proof}
\label{lem_main_value_recovery}
\end{lem}

Finally, combining these results we get Theorem \ref{thm_recovery_improvement} which shows that if the true value of the recovery policy improves, the true value of the main policy is guaranteed to improve as well.

\improvementrecoverytheorem*
\begin{proof}
From Lemma \ref{lem_main_value_recovery}, we have the value of the main policy $V^{\pi}(s)+V^{\bar{\pi}}(s)=G(s)+V^{R}(s)$ and the value of the updated policy induced by the improved recovery policy given as $V^{\pi'}(s)+V^{\bar{\pi}'}(s)=G(s)+V^{R'}(s)$.  Since $V^{R'}(s) \geq V^{R}(s)$ for all $s \in S$, we have the following
\begin{equation}
    \begin{split}
        V^{R'}(s) - V^{R}(s) & \geq 0\\
        V^{\pi'}(s) + V^{\bar{\pi}'} - V^{\pi}(s) - V^{\bar{\pi}} & \geq 0\\
        V^{\pi'}(s) - V^{\pi}(s) & \geq V^{\bar{\pi}}(s) - V^{\bar{\pi}'}(s)
    \end{split}
    \label{eq_main_improve_inequality_recovery}
\end{equation}
Now suppose that there exists a states $s$ such that $V^{\pi}(s) > V^{\pi'}(s)$ for any state $s \in S$.  We will try to show by contradiction that this cannot be true.  If $V^{\pi}(s) > V^{\pi'}(s)$, then $V^{\pi'}(s) - V^{\pi}(s) < 0$ and it follow from the inequality in equation \eqref{eq_main_improve_inequality_recovery} that $V^{\bar{\pi}}(s) - V^{\bar{\pi}'}(s) < 0$.  In addition, we can use Lemma \ref{lem_main_value_recovery} to write $V^{\pi}(s) > V^{\pi'}(s)$ as
\begin{equation}
    \begin{split}
        V^{\pi}(s) & > V^{\pi'}(s)\\
        V^{R}(s) - V^{\bar{\pi}}(s) & > V^{R'}(s) - V^{\bar{\pi}'}(s)\\
        V^{\bar{\pi}'}(s) - V^{\bar{\pi}}(s) & > V^{R'}(s) - V^{R}(s)\\
            & > 0
    \end{split}
\end{equation}
However, $V^{\bar{\pi}'}(s) - V^{\bar{\pi}}(s) > 0$ contradicts that $V^{\bar{\pi}}(s) - V^{\bar{\pi}'}(s) < 0$ and thus all states $s \in S$ must satisfy $V^{\pi'}(s) 
\geq V^{\pi}(s)$.
\end{proof}

In order to support our discussion and understanding of the impact learning has on the value of the main policy, it is easy to show that an improvement in the recovery policy leads to a contraction of the initiation set $L$ given in Lemma \ref{lem_recovery_improvement_contracts_L}.  The implication is that learning will result in reducing the dependency on the primal option as the recovery option improves in a manner similar to curriculum learning.

\begin{lem}[Improving Recovery Contracts $L$]
If the recovery policy $\pi^{R}(a|s)$ is improved such that $V^{R'}(s) \geq V^{R}(s)$ then the new $\tau(s)=V^{R'}(s)$ results in a new initiation set $L'$ such that $L' \subseteq L$.
\begin{proof}
For any state $s \in L'$, $G(s) \geq V^{R'}(s)$ according to Definition \ref{def_initiation_set}; in addition, since $V^{R'}(s) \geq V^{R}(s)$, then $G(s) \geq V^{R}(s)$ and $s \in L$.  Thus any state $s \in L'$ must also be in $L$.  Alternatively, suppose that there exists a state $s \in L'$ that is not in $L$, then $G(s) < V^{R}(s)$ and $G(s) \geq V^{R'}(s)$ according to Definition \ref{def_initiation_set}. Since $V^{R'}(s) \geq V^{R}(s)$, this results in a contradiction.  Thus $L' \subseteq L$.
\end{proof}
\label{lem_recovery_improvement_contracts_L}
\end{lem}
\subsection{Student Option Proofs}

In this section, we prove the theorems and lemmas concerning the learning of student options as defined in Definition \ref{def_student_option}.  We start first with the proof of optimality which is straight forward to show assuming that student policy learning is optimal.  While such an assumption is strong, the point is that the LISPR framework guarantees optimality for the overall main policy under such an assumption.

\optimalstudenttheorem*
\begin{proof}
Let us assume that the student policy $\pi^{S*}(a|s)$ is optimal such that the value $V^{S*}(s)=\max_{\forall \pi}{V^{\pi}(s)}$ for all states $s \in S$.  We define the threshold $\tau(s)=V^{S*}(s)$ for all states $s \in S$.  Let us consider the following cases:
\setcounter{casep}{0}
\begin{casep}[$s_t \in L$]
\label{student_optimality_case_s_in_L}
From Definition \ref{def_initiation_set} we have the condition that $G(s_t) \geq V^{S*}(s_t)$ for all states $s_t \in L$.  However, due to $\pi^{S*}(a|s)$ being optimal, we also have $G(s_t) \leq V^{S*}(s_t)$.  Therefore, $G(s_t)=V^{S*}(s_t)$ for all states $s_t \in L$.  This means that the primal option is optimal if $s_t \in L$.  Furthermore, the value of the main policy is given by
\begin{equation}
    \begin{split}
        V^{\pi}(s_t)  & = \E_{\mu} \left[ \sum_{i=0}^{\infty}{\gamma^i r_{t+i+1}} \right]\\
                    & = G(s_t)
    \end{split}
\end{equation}
\end{casep}
\begin{casep}[$s_k \notin L$ for all $k \geq t$]
From the main policy definition, we have $\pi(s_t)=o^S$ for all states $s_t \notin L$.  Therefore, the value of the main policy is given by 
\begin{equation}
    \begin{split}
        V^{\pi}(s_t)  & = \E_{\pi^{S*}} \left[ \sum_{i=0}^{\infty}{\gamma^i r_{t+i+1}} \right]\\
                    & = V^{S*}(s_t)
    \end{split}
\end{equation}
\end{casep}
\begin{casep}[$s_k \notin L$ for all $t \leq k < T$ where $s_T \in L$]
Let us consider a future state $s_T$ for $T>t$ where $s_T \in L$ such that all states $s_k \notin L$ from $k=t ... T-1$, then $V^\pi(s_t)$ can be written as
\begin{equation}
    \begin{split}
        V^{\pi}(s_t)  & = \E_{\pi} \left[ \sum_{i=0}^{\infty}{\gamma^i r_{t+i+1}} \right]\\
                    & = \E_{\pi^{S*}} \left[ \sum_{i=0}^{T-t-1}{\gamma^i r_{t+i+1}} \right] + \gamma^{T-t} V^{\pi}(s_T)\\
                    & = V^{S*}(s_t) - \gamma^{T-t} V^{S*}(s_T) + \gamma^{T-t} V^{\pi}(s_T)\\
                    & = V^{S*}(s_t) + \gamma^{T-t} \left( V^{\pi}(s_T) - V^{S*}(s_T) \right)
    \end{split}
    \label{eq_student_value_s_notin_L}
\end{equation}

Since $s_T \in L$, we can use \cref{student_optimality_case_s_in_L} to get 
\begin{equation}
    \begin{split}
        V^{\pi}(s_t) & = V^{S*}(s_t) + \gamma^{T-t} \left( G(s_T) - G(s_T) \right)\\
                    & = V^{S*}(s_t)
    \end{split}
\end{equation}
\end{casep}

Taking the three cases together, we have $V^{\pi}(s)=V^{S*}(s_t)$ for all states $s_t \in S$.  Therefore, $\pi(s)$ in Definition \ref{def_main_policy} is optimal if $\pi^{S*}(a|s)$ is optimal.
\end{proof}

Next, we introduce a notation for finite horizon value that will be useful in future proofs.

\begin{defn}[Finite Horizon for Student Options]
The finite horizon value for trajectory of states $s_t ... s_{T-1}$ under the student option is given by 
\begin{equation}
    \begin{split}
        H_T^{S}(s_t) & = \E_{\pi^{S}} \left[ \sum_{i=0}^{T-1}{\gamma^i r_{t+i+1}} \right]\\
            & = V^{S}(s_t) - \gamma^{T-t} V^{S}(s_T)
    \end{split}
    \label{eq_finite_horizon_student}
\end{equation}
\label{def_finite_horizon_student}
\end{defn}

To prove the improvement Theorem \ref{thm_student_improvement}, we need Lemmas \ref{lem_finite_horizon_student_inequality_s_in_L} and \ref{lem_finite_horizon_student_inequality_s_notin_L} below.

\begin{lem}[Finite Horizon Inequality Starting With the Primal Option]
If there is a trajectory of states from $s_t \in L$ to $s_T \notin L$ such that $G(s_T) < V^{S}(s_T)$, where $s_k \in L$ for $k=t...T-1$, then $H_T(s_t) \geq H_T^{S}(s_t)$.
\begin{proof}
\begin{equation}
    \begin{split}
        G(s_t) & \geq V^{S}(s_t)\\
        G(s_t) - \gamma^{T-t} G(s_T) & \geq V^{S}(s_t) - \gamma^{T-t} V^{S}(s_T)\\
        H_T(s_t) & \geq H_T^{S}(s_t)
    \end{split}
    \label{eq_finite_horizon_student_inequality_s_in_L}
\end{equation}
\end{proof}
\label{lem_finite_horizon_student_inequality_s_in_L}
\end{lem}

\begin{lem}[Finite Horizon Inequality Starting with the Student Policy]
If there is a trajectory of states from $s_t \notin L$ to $s_T \in L$, where $s_k \notin L$ for $k=t...T-1$, then $H_T(s_t) < H_T^{S}(s_t)$.
\begin{proof}
\begin{equation}
    \begin{split}
        G(s_t) & < V^{S}(s_t)\\
        G(s_t) - \gamma^{T-t} G(s_T) & < V^{S}(s_t) - \gamma^{T-t} V^{S}(s_T)\\
        H_T(s_t) & < H_T^{S}(s_t)
    \end{split}
    \label{eq_finite_horizon_student_inequality_s_notin_L}
\end{equation}
\end{proof}
\label{lem_finite_horizon_student_inequality_s_notin_L}
\end{lem}

We will now prove Lemma \ref{lem_student_lower_bound}, which says that the performance of the student policy has a lower bound.

\boundstudentlemma*
\begin{proof}
Consider the following cases given any state $s_t \in S$:
\setcounter{casep}{0}
\begin{casep}[Trajectory starting in state $s_t \in L$]
Let us denote times where the main policy changes between ``source'' and ``student'' as $T_{j+1} > T_{j} > t$ for $j=1 ... N$ where $N$ is the number of total changes starting from state $s_t$.  We can write the value of the main policy in terms of finite horizons given in Definitions \ref{def_finite_horizon_source} and \ref{def_finite_horizon_student}.  Given Lemmas \ref{lem_finite_horizon_student_inequality_s_in_L} and \ref{lem_finite_horizon_student_inequality_s_notin_L}, the main policy value function can be written as
\begin{equation}
    \begin{split}
        V^{\pi}(s_t) & = H_{T_1}(s_t) + \gamma^{T_1-t} H_{T_2}^{S}(s_{T_1}) + \gamma^{T_2-t} H_{T_3}(s_{T_2}) + ...\\
                & \geq H_{T_1}(s_t) + \gamma^{T_1-t} H_{T_2}(s_{T_1}) + \gamma^{T_2-t} H_{T_3}(s_{T_2}) + ...\\
                & \geq G(s_t)
    \end{split}
    \label{eq_vpi_finite_horizon_student_st_in_L}
\end{equation}
Since $s_t \in L$, we also have $G(s_t) \geq V^{S}(s_t)$ by Definition \ref{def_initiation_set}.  Therefore, $V^{\pi}(s_t) \geq G(s_t) \geq V^{S}(s_t)$ for any trajectory starting in state $s_t \in L$.
\end{casep}
\begin{casep}[Trajectory starting in state $s_t \notin L$]
Let us denote times where the main policy changes between ``source'' and ``student'' as $T_{j+1} > T_{j} > t$ for $j=1 ... N$ where $N$ is the number of total changes starting from state $s_t$.  We can write the value of the main policy in terms of finite horizons given in Definitions \ref{def_finite_horizon_source} and \ref{def_finite_horizon_student}.  Given Lemmas \ref{lem_finite_horizon_student_inequality_s_in_L} and \ref{lem_finite_horizon_student_inequality_s_notin_L}, the main policy value function can be written as
\begin{equation}
    \begin{split}
        V^{\pi}(s_t) & = H_{T_1}^{S}(s_t) + \gamma^{T_1-t} H_{T_2}(s_{T_1}) + \gamma^{T_2-t} H_{T_3}^{S}(s_{T_2}) + ...\\
                & \geq H_{T_1}^{S}(s_t) + \gamma^{T_1-t} H_{T_2}^{S}(s_{T_1}) + \gamma^{T_2-t} H_{T_3}^{S}(s_{T_2}) + ...\\
                & \geq V^{S}(s_t)
    \end{split}
    \label{eq_vpi_finite_horizon_student_st_not_in_L}
\end{equation}
Since $s_t \notin L$, we also have $G(s_t) < V^{S}(s_t)$ by Definition \ref{def_initiation_set}.  Therefore, $V^{\pi}(s_t) \geq V^{S}(s_t) > G(s_t)$ for any trajectory starting in state $s_t \notin L$.
\end{casep}
Combining these cases, we conclude that $V^{\pi}(s) \geq V^{S}(s)$ and $V^{\pi}(s) \geq G(s)$ for all $s \in S$.
\end{proof}

We go one step further to define the value of the main policy according to $G(s)$ and $V^{S}(s)$.

\begin{lem}[Main Policy Value with Student Option]
The value of the main policy $V^{\pi}(s)$ is equal to $V^{\pi}(s)+V^{\bar{\pi}}(s)=G(s)+V^{S}(s)$ given student policy $\pi^{S}(a|s)$ where $\bar{\pi}(s)$ is the ``anti-main'' policy where the decision rule in Equation \ref{eq_main_policy} is reversed such that $o^P$ is chosen if $s \notin L$ and $o^S$ otherwise.
\begin{proof}
Let us consider the following cases:
\setcounter{casep}{0}
\begin{casep}[State $s \in L$]
We can write $V^{\pi}(s_t) - V^{S}(s_t)$ as a function of finite horizons:
\begin{equation}
    \begin{split}
        V^{\pi}(s_t) - V^{S}(s_t) = & H_{T_1}(s_t) + \gamma^{T_2-t} H_{T_3}(s_{T_2}) + ...\\
            & - (H_{T_1}^{S}(s_t) + \gamma^{T_2-t}H_{T_3}^{S}(s_{T_2}) + ...)
    \end{split}    
\end{equation}
where $t \leq T_1 \leq T_2 \leq ...$ such that $s_t \in L, s_{T_1} \notin L, s_{T_2} \in L, ...$.  Similar we can write $V^{\pi}(s_t) - G(s_t)$ as a function of finite horizons:
\begin{equation}
    \begin{split}
        V^{\pi}(s_t) - G(s_t) = & \gamma^{T_1-t} H_{T_2}^{S}(s_{T_1}) + \gamma^{T_3-t} H_{T_4}^{S}(s_{T_3}) + ...\\
            & - (\gamma^{T_1-t} H_{T_2}(s_{T_1}) + \gamma^{T_3-t} H_{T_4}(s_{T_3}) + ...)
    \end{split}
\end{equation}
Adding $V^{\pi}(s_t) - V^{S}(s_t)$ and $V^{\pi}(s_t) - G(s_t)$ together, gives us 
\begin{equation}
    \begin{split}
        2V^{\pi}(s_t) - G(s_t) - V^{S}(s_t) & = V^{\pi}(s_t) - V^{\bar{\pi}}(s_t)\\
        V^{\pi}(s_t) + V^{\bar{\pi}}(s_t) & = G(s_t) + V^{S}(s_t)
    \end{split}
\end{equation}
\end{casep}
\begin{casep}[State $s \notin L$]
Similar to the previous case, we can write $V^{\pi}(s_t) - V^{S}(s_t)$ as a function of finite horizons:
\begin{equation}
    \begin{split}
        V^{\pi}(s_t) - V^{S}(s_t) = & \gamma ^{T_1-t}H_{T_2}(s_{T_1}) + \gamma^{T_3-t} H_{T_4}(s_{T_3}) + ...\\
            & - (\gamma ^{T_1-t}H_{T_2}^{S}(s_{T_1}) + \gamma^{T_3-t} H^{S}_{T_4}(s_{T_3}) + ...)
    \end{split}
\end{equation}
where $t \leq T_1 \leq T_2 \leq ...$ such that $s_t \notin L, s_{T_1} \in L, s_{T_2} \notin L, ...$.  Similarly we can write $V^{\pi}(s_t) - G(s_t)$ as a function of finite horizons of the student policy:
\begin{equation}
    \begin{split}
        V^{\pi}(s_t) - G(s_t) = & H_{T_1}^{S}(s_t) + \gamma^{T_2-t}H_{T_3}^{S}(s_{T_2}) + ...\\
            & - (H_{T_1}(s_t) + \gamma^{T_2-t}H_{T_3}(s_{T_2}) + ...)
    \end{split}
\end{equation}
Adding $V^{\pi}(s_t) - V^{S}(s_t)$ and $V^{\pi}(s_t) - G(s_t)$ together, gives us 
\begin{equation}
    \begin{split}
        2V^{\pi}(s_t) - G(s_t) - V^{S}(s_t) & = V^{\pi}(s_t) - V^{\bar{\pi}}(s_t)\\
        V^{\pi}(s_t) + V^{\bar{\pi}}(s_t) & = G(s_t) + V^{S}(s_t)
    \end{split}
\end{equation}
\end{casep}
Therefore, for all states $s \in S$, we have $V^{\pi}(s)+V^{\bar{\pi}}(s)=G(s)+V^{S}(s)$.
\end{proof}
\label{lem_main_value_student}
\end{lem}

Finally, combining these results we get Theorem \ref{thm_student_improvement} which shows that if the true value of the student policy improves, the true value of the main policy is guaranteed to improve as well.

\improvementstudenttheorem*
\begin{proof}
From Lemma \ref{lem_main_value_student}, we have the value of the main policy $V^{\pi}(s)+V^{\bar{\pi}}(s)=G(s)+V^{S}(s)$ and the value of the updated policy induced by the improved student policy given as $V^{\pi'}(s)+V^{\bar{\pi}'}(s)=G(s)+V^{S'}(s)$.  Since $V^{S'}(s) \geq V^{S}(s)$ for all $s \in S$, we have the following
\begin{equation}
    \begin{split}
        V^{S'}(s) - V^{S}(s) & \geq 0\\
        V^{\pi'}(s) + V^{\bar{\pi}'} - V^{\pi}(s) - V^{\bar{\pi}} & \geq 0\\
        V^{\pi'}(s) - V^{\pi}(s) & \geq V^{\bar{\pi}}(s) - V^{\bar{\pi}'}(s)
    \end{split}
    \label{eq_main_improve_inequality_student}
\end{equation}
Now suppose that there exists a states $s$ such that $V^{\pi}(s) > V^{\pi'}(s)$ for any state $s \in S$.  We will try to show by contradiction that this cannot be true.  If $V^{\pi}(s) > V^{\pi'}(s)$, then $V^{\pi'}(s) - V^{\pi}(s) < 0$ and it follows from the inequality in equation \eqref{eq_main_improve_inequality_student} that $V^{\bar{\pi}}(s) - V^{\bar{\pi}'}(s) < 0$.  In addition, we can use Lemma \ref{lem_main_value_student} to write $V^{\pi}(s) > V^{\pi'}(s)$ as
\begin{equation}
    \begin{split}
        V^{\pi}(s) & > V^{\pi'}(s)\\
        V^{S}(s) - V^{\bar{\pi}}(s) & > V^{S'}(s) - V^{\bar{\pi}'}(s)\\
        V^{\bar{\pi}'}(s) - V^{\bar{\pi}}(s) & > V^{S'}(s) - V^{S}(s)\\
            & > 0
    \end{split}
\end{equation}
However, $V^{\bar{\pi}'}(s) - V^{\bar{\pi}}(s) > 0$ contradicts that $V^{\bar{\pi}}(s) - V^{\bar{\pi}'}(s) < 0$ and thus all states $s \in S$ must satisfy $V^{\pi'}(s) 
\geq V^{\pi}(s)$.
\end{proof}

In order to support our discussion and understanding of the impact learning has on the value of the main policy, it is easy to show that an improvement in the student policy leads to a contraction of the initiation set $L$ given in Lemma \ref{lem_student_improvement_contracts_L}.  The implication is that learning will result in reducing the dependency on the primal option as the recovery option improves in a manner similar to curriculum learning.

\begin{lem}[Improving Student Contracts $L$]
If the student policy $\pi^{S}(a|s)$ is improved such that $V^{S'}(s) \geq V^{S}(s)$ then the new $\tau(s)=V^{S'}(s)$ results in a new initiation set $L'$ such that $L' \subseteq L$.
\begin{proof}
Since state $s \in L'$, thus $G(s) \geq V^{S'}(s)$; however, since $V^{S'}(s) \geq V^{S}(s)$, then $G(s) \geq V^{S}(s)$ and $s \in L$.  Thus all states $s \in L'$ must also be in $L$.  Now suppose that there exists a state $s \in L'$ that is not in $L$, then $G(s) < V^{S}(s)$ and $G(s) \geq V^{S'}(s)$ according to Definition \ref{def_initiation_set} which results in a contradiction.  Thus $L' \subseteq L$.
\end{proof}
\label{lem_student_improvement_contracts_L}
\end{lem}

\subsection{Algorithms}
\begin{algorithm}
\caption{Learning a Student Policy with LISPR}
\label{alg_lispr_student}
\begin{algorithmic}[1]
\State Initialize replay memory $D$, $G(s,a)$
\For{i=0,N}
    \State Observe initial state $s_0$
    \For{t=0,T}
        \State Choose an option $o_t$ according to the main policy with probability $1-\epsilon$ otherwise choose randomly
        \If{$o_t$ is Primal $o^P$}
            \State Sample action $a_t$ according to $\mu(a_t|s_t)$
        \Else
            \State Sample action $a_t$ according to $\pi^{S}(a_t|s_t)$
        \EndIf
        \State Execute action $a_t$
        \State Observe state $s_{t+1}$, reward $r_{t+1}$ and done flag
        \If{done}
            \State $\gamma_{t+1} = 0$
        \Else
            \State $\gamma_{t+1} = \gamma$
        \EndIf
        \State Store transition $(s_t,o_t,a_t,r_{t+1},\gamma_{t+1},s_{t+1})$ in $D$
        \State Sample random minibatch $B$ from $D$
        \State Update $G(s_i, a_i;\theta)$ for minibatch $B$ according to \eqref{eq_success_update}
        \State Update $\pi^{S}(a_i|s_i)$ and $Q^{S}(s_i, a_i)$ for minibatch $B$ according to any off-policy RL algorithm
    \EndFor
\EndFor
\end{algorithmic}
\end{algorithm}

The update for $G(s,a;\theta)$ parameterized by $\theta$ is an off-policy update using temporal difference learning given by the loss function $L_G(\theta)=\frac{1}{2}\E_{D}{\left[ (\delta_i)^2 \right]}$ where $\delta_i = G(s_i, a_i; \theta) - y_i$, the bootstrapped prediction is $y_i=r_{i+1} + \gamma_{i+1} G(s_{i+1}, \hat{a}_{i+1})$ and $\hat{a}_{i+1}$ is the action sampled from the source policy for the next state, i.e. $\hat{a}_{i+1} \sim \mu(a_{i+1}|s_{i+1})$.
The gradient of $L_G(\theta)$ with respect to $\theta$ is thus given by 

\begin{equation}
    \nabla_{\theta}{L_G(\theta)} = \E_{D}{\left[ \delta_i\nabla_{\theta}{G(s_i,a_i)} \right]}
    \label{eq_success_update}
\end{equation}

The proposed algorithm would benefit from methods that reduce the maximization bias such as double clipped Q-learning and others to get better early estimates for $Q^{S}(s,a)$ which results in better utilization of the primal option early in learning.  We had more success estimating using $\tau(s)=V^{\pi}(s)$ as the threshold in the main policy and approximating $V^{\pi}(s)$ with a biased estimate $V^{\hat{\pi}}(s)$ where $\hat{\pi}$ is the behavior policy which includes exploration.  $V^{\hat{\pi}}(s)$ was approximated using samples from the replay buffer.  It is not clear how important this bias is as long as $G(s) \geq V^\pi(s)$ when the source policy is optimal and $G(s) < V^\pi(s)$ otherwise.
The situation merits further investigation on less biased estimates of $V^\pi(s)$.

An algorithm for learning recovery policy is also provided as Algorithm \ref{alg_lispr_recovery}. Recovery policy learning is more complicated and computationally less efficient than student policy learning. Algorithm \ref{alg_lispr_recovery} uses off-policy methods. Since the recovery policy terminates when switching from recovery to source, however, on-policy RL algorithms could also be suitable.

\begin{algorithm}
\caption{Learning a Recovery Policy with LISPR}
\label{alg_lispr_recovery}
\begin{algorithmic}[1]
\State Initialize replay memory $D$, $G(s,a)$
\For{i=0,N}
    \State Observe initial state $s_0$
    \For{t=0,T}
        \State Choose an option $o_t$ according to the main policy with probability $1-\epsilon$ otherwise choose randomly
        \If{$o_t$ is Primal $o^P$}
            \State Sample action $a_t$ according to $\mu(a_t|s_t)$
        \Else
            \State Sample action $a_t$ according to $\pi^{R}(a_t|s_t)$
        \EndIf
        \State Execute action $a_t$
        \State Observe state $s_{t+1}$, reward $r_{t+1}$ and done flag
        \If{done}
            \State $\gamma_{t+1} = 0$
        \Else
            \State $\gamma_{t+1} = \gamma$
        \EndIf
        \State Store transition $(s_t,o_t,a_t,r_{t+1},\gamma_{t+1},s_{t+1})$ in $D$
        \State Sample random minibatch $B$ from $D$
        \State Update $G(s_i, a_i;\theta)$ for minibatch $B$ according to \eqref{eq_success_update}
        \For{each $(s_i,o_i,a_i,r_{i+1},\gamma_{i+1},s_{i+1}) \in B$}
            \State Choose $o_{i+1}$ according to the main policy for state $s_{i+1}$
            \If{$o_{i+1}$ is the recovery option $o^R$}
                \State Set $\gamma_{i+1} = 0$
                \State Compute $r_{i+1}=G(s_i, a_i)$
            \EndIf
        \EndFor
        \State Store modified transitions in $A$
        \State Update $\pi^{R}(a_i|s_i)$ and $Q^{R}(s_i, a_i)$ for $A$ according to any RL algorithm
    \EndFor
\EndFor
\end{algorithmic}
\end{algorithm}

Both Algorithms \ref{alg_lispr_student} and \ref{alg_lispr_recovery} will benefit from methods that reduce the early biased estimates of the value function $V^{S}(s)$ and $V^{R}(s)$ respectively.  Our approach of estimating the value of the behavior policy during learning was effective but certainly may not be the best approach.

\subsection{Experiment Details}
Here we describe the details of the experiments including the hyper-parameters used and their parameter sweeps.

\subsubsection{Tabular Tasks}
For all tabular experiments, $\gamma=0.99$ (although $\gamma=1$ also worked fine since the problem is episodic), where the learning rate and $\epsilon$ were annealed.
Q-learning with $\lambda$-trace was used for learning, where the parameter sweeps were ${0, 0.2, 0.4, 0.6, 0.8, 1}$ for $\lambda$, ${0.1, 0.25, 0.5}$ for learning rate $\alpha$, and ${0.5, 0.75, 1.0}$ for the initial $\epsilon$ value.
The exploration of the main policy was implemented with $\epsilon$-greedy as well and swept over the set ${0.25, 0.5, 0.75, 1.0}$.
The parameters with a final performance of 1 and maximum area under the curve where chosen as optimal values.
The tuned parameters for \textbf{multiroom world} are given in Table \ref{table_room_world_parameters}.

\begin{table}
    \caption{Tuned parameters for \textbf{multiroom world}}
    \centering
    \begin{tabular}{|c c c|} 
        \hline
        Parameters & Recovery & Baseline \\ 
        \hline\hline
        $\alpha$ & 0.25 & 0.5 \\ 
        \hline
        $\lambda$-trace & 0.0 & 0.6 \\
        \hline
        $\epsilon$-greedy & 1.0 & 1.0 \\
        \hline
        Main $\epsilon$-greedy & 0.25 & - \\
        \hline
        Max Steps & 500,000 & 500,000 \\
        \hline
    \end{tabular}
    \label{table_room_world_parameters}
\end{table}

The tuned parameters for \textbf{box world} are given in Table \ref{table_box_world_parameters}.

\begin{table}
    \caption{Tuned parameters for \textbf{box world}}
    \centering
    \begin{tabular}{|c c c|} 
        \hline
        Parameters & Recovery & Baseline \\ 
        \hline\hline
        $\alpha$ & 0.25 & 0.1 \\ 
        \hline
        $\lambda$-trace & 1.0 & 0.6 \\
        \hline
        $\epsilon$-greedy & 1.0 & 1.0 \\
        \hline
        Main $\epsilon$-greedy & 0.5 & - \\
        \hline
        Max Steps & 2,000,000 & 2,000,000 \\
        \hline
    \end{tabular}
    \label{table_box_world_parameters}
\end{table}

In all tabular experiments, the $\epsilon$ of the Q-learning agent was annealed from its initial value to a final value of 0.1 when the max number of steps was reached.

\subsubsection{Social Navigation Task}
Anywhere from two to three pedestrians are generated at random positions in the map.
Each pedestrian had random goals in the map while avoiding other pedestrians.
The pedestrians ignore the agent, effectively always asserting the right of way in relation to the agent. This results in many collisions that the agent must learn to avoid.
Figure \ref{fig_toy_example} shows an example scenario with two pedestrians.  

\begin{figure}
    \centering
    \includegraphics[width=0.45\textwidth]{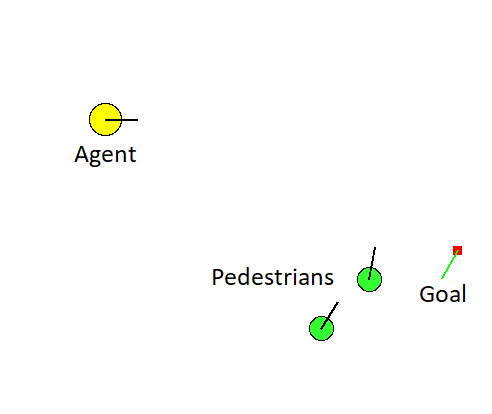}
    \caption{An example of a social navigation scenario.}
    \label{fig_toy_example}
\end{figure}

The observation of the agent is a 2D lidar scan and relative goal position for the agent $g$.
A history of 20 frames is supplied to the agent.
The action space is a velocity vector $v$.
The source policy is a heuristic policy given by $v=\frac{g}{|g|}v^{max}$ where $v^{max}=1.0$ is the maximum speed of the agent.
The time step in the simulator is 0.05 seconds.
The heuristic policy does not perform well among pedestrians and achieves success on average 67\% and collision the remaining 33\% of the time.
The average return for the heuristic policy is $0.0938$. 

The DDPG exploration is an Ornstein-Uhlenbeck noise process with $\theta=0.15$, $\sigma=0.2$ with $\delta t = 1.0$.
The same exploration process is used for learning the student policy with LISPR and learning tabula rasa.
The actor and critic used a target network smoothing factor of $\tau=0.01$ for both networks.  The actor and critic learning rates were both $\alpha=0.0001$.
The discount was $\gamma=0.99$.
DDPG is implemented with double clipped Q-learning as with TD3 where every second actor update is skipped.
The replay buffer had a capacity of 100,000 transitions and the minibatch size was 32.
An update was performed every two simulator steps for a total of 1 million updates (i.e. 2 million environment steps).
For LISPR with DDPG, the same parameters are used.
In addition, 1000 samples are collected from the environment using the source option to prime the replay buffer with good samples to help learn a good initial estimate for $G(s)$.
And a tolerance of 0.001 was used when checking for equality between $G(s)$ and $V^{\pi}(s)$ for the main policy decision rule.
The main policy $\epsilon$-greedy was $\epsilon=0.2$.
While many different values were tried, it was found that the value did not have a large effect on the performance of the agent.
The results with DDPG are shown in Figure \ref{fig_toy_ddpg}.
The Ornstein-Uhlenbeck exploration process was tuned to achieve the results depicted where there is only a slight advantage with LISPR over DDPG.
The advantage of LISPR is that all 10 training runs converged to a good solution; however, DDPG failed to converge in some training runs where DDPG learned a policy that remained stationary.
Thus LISPR was found to be more consistent in learning.
It is clear that with the right hand-engineered exploration policy, one can achieve competitive performance with LISPR although LISPR may converge to the optimal solution more consistently.

\begin{figure}
    \centering
    \includegraphics[width=0.45\textwidth]{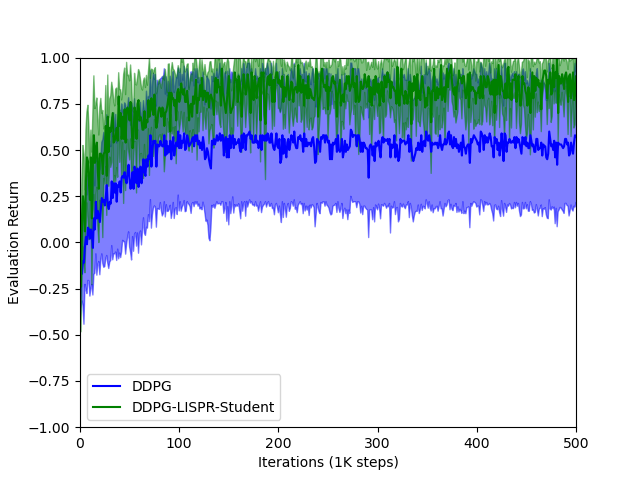}
    \caption{DDPG versus LISPR with DDPG on the \textbf{social navigation task}}
    \label{fig_toy_ddpg}
\end{figure}

The SAC agent was trained in the same way as DDPG except the exploration is learned by learning $\alpha$.
The initial value of $\alpha$ was 0.1.
Double clipped Q-learning was used and a target network smoothing value of $\tau=0.1$ was used.
All learning rates were 0.0001.
It was found that SAC could tolerate larger values of target network smoothing for faster learning.
The same parameters are used for source option with SAC but like DDPG, 1000 samples are collected from the environment using the source option to prime the replay buffer.
Similar experiments were conducted using the same reward shaping as \citep{jinj2020socialnavigation} and the results showed that DDPG learned nearly the same as LISPR with DDPG and that LISPR with SAC learned a bit faster.
The neural network architectures for both DDPG and SAC were the same as in \citep{jinj2020socialnavigation}.



\subsubsection{Bipedal Walker Task}
This task involves transfer from the popular BipedalWalker task to the BipedalWalkerHardcore task in Gym.
The DDPG exploration is an Ornstein-Uhlenbeck noise process with $\theta=0.15$, $\sigma=0.3$ with $\delta t = 1.0$.
The same exploration process is used for learning the student policy with LISPR as it is for learning tabula rasa.
It was tuned specifically for the source task on BipedalWalker.
The actor and critic used a target network smoothing factor of $\tau=0.01$ and $\tau=0.005$ respectively.
The actor and critic learning rates were $\alpha=0.0001$ and $\alpha=0.001$ respectively.
The discount was $\gamma=0.99$.
The DDPG implemented used double clipped Q-learning similar to TD3 and every second actor update was skipped.
The replay buffer contained 1,000,000 transitions and minibatch size was 32.
An update was performed every two simulator steps for a total of 2 million updates (i.e. 4 million environment steps).
For source option with DDPG as well as progressive networks with DDPG, the same parameters are used.
In addition, 10,000 samples are collected from the environment using the source option to prime the replay buffer although this was not found to be a critical parameter.
A tolerance of 0.001 was used when checking for equality between $G(s)$ and $V^{\pi}(s)$ for the main policy decision rule.
The exploration of the main policy was done using $\epsilon$-greedy where $\epsilon=0.2$.
Other values of $\epsilon$ did not have a large effect on performance although when $\epsilon>0.5$, performance dropped a bit.
Poor learning was observed if $\epsilon=0$ thus some exploration is needed.


The SAC agent was trained in the same way as DDPG except the exploration is learned by learning $\alpha$.
The initial value of $\alpha$ was $\frac{1}{300}$ since the reward was scaled by the same factor.
Double clipped Q-learning was used and a target network smoothing value of $\tau=0.1$ was used.
All learning rates were 0.0001.
It was found that SAC could tolerate larger values of target network smoothing $\tau$ for faster learning.
The same parameters are used for LISPR with SAC but like with DDPG, 1000 samples are collected from the environment using the LISPR to prime the replay buffer.

The neural network architecture for the actor and critic were identical for both DDPG and SAC with the only exception being the number of outputs.
All networks consisted of 3 fully connected layers with rectified linear activation for all layers except the last layer.
The first layer consisted of 400 linear units and the second layer consisted of 300 linear units.
The output layer of both the DDPG and SAC critic was scalar with linear activation.
The input to the DDPG and SAC critic was the state and action which is 24 and 4 dimensional vectors respectively.
The DDPG actor had 4 outputs with hyperbolic tangent activation.
The SAC actor had 8 outputs with linear activation outputting the mean and standard deviation for each action.
A hyperbolic tangent function was applied to the SAC action output.

\subsubsection{Lunar Lander Task}
Learning the lunar lander Gym task with LISPR results in the agent solving the task more consistently during training.
A heuristic policy was used as the source policy.
This was the same heuristic policy implemented in the gym environment source code.
It was found that the heuristic policy was usually slightly better than the final learned policies; thus, the policy was crippled to mimic an initiation set that needed to be learned by applying a poor default action of full main engine throttle if the agent was too far left or right of the landing pad.

The same exploration process used in BipedalWalker was used in LunarLander for DDPG.
It was observed that performance was only slightly sensitive to the exploration parameters.
The learning rates were all 0.0001 and the target network smoothing parameters were all $\tau=0.1$.
The replay buffer had a capacity of 1 million samples and a minibatch size of 128 was used.
Updates were applied after every 2 steps in the environment for a total of 1 million updates (i.e. a total of 2 million samples).
Very similar parameters were used with SAC except the exploration parameter $\alpha$ was set to $\frac{1}{1000}$ since the reward was scaled by the same amount.
When learning with LISPR, a fixed threshold of $\tau(s)=0.2$ was used to understand the performance with fixed thresholds.
The result is shown in Figures \ref{fig_lunar_ddpg} and \ref{fig_lunar_sac} are averaged over 3 runs.
We see more consistent and slightly faster learning for the LISPR framework when using both DDPG and SAC.
The same neural network architecture used for bipedal walker was used for lunar lander.

\begin{figure}
    \centering
    \includegraphics[width=0.45\textwidth]{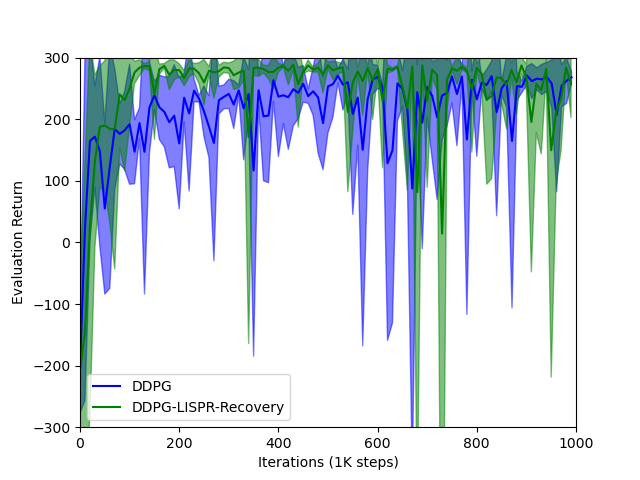}
    \caption{DDPG versus LISPR with DDPG on the \textbf{lunar lander task} with 90\% confidence interval}
    \label{fig_lunar_ddpg}
\end{figure}

\begin{figure}
    \centering
    \includegraphics[width=0.45\textwidth]{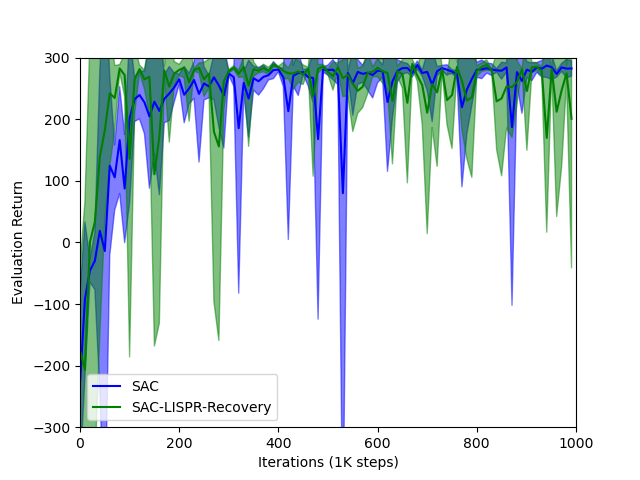}
    \caption{SAC versus LISPR with SAC on the \textbf{lunar lander task} with 90\% confidence interval}
    \label{fig_lunar_sac}
\end{figure}

\end{document}